%% file: arxiv.tex
\documentclass{article}
\usepackage{arxiv}
 \pdfoutput=1

\usepackage{matharticle}
\usepackage[utf8]{inputenc} 
\usepackage[T1]{fontenc}    
\usepackage{amsfonts}       
\usepackage{nicefrac}       
\usepackage{microtype}      
\usepackage{booktabs} 
\usepackage{array}
\usepackage{layouts}
\usepackage{floatrow}
\usepackage{float}

\usepackage{amsmath,amsthm,amssymb}
\newtheorem{theorem}{Theorem}

\newtheorem{lemma}[theorem]{Lemma}

\newtheorem{definition}[theorem]{Definition}

\newtheorem*{rep@theorem}{\rep@title}
\newcommand{\newreptheorem}[2]{%
\newenvironment{rep#1}[1]{%
 \def\rep@title{#2 \ref{##1}}%
 \begin{rep@theorem}}%
 {\end{rep@theorem}}}
\newreptheorem{theorem}{Theorem}
\newreptheorem{lemma}{Lemma}
\newreptheorem{proposition}{Proposition}
\newreptheorem{claim}{Claim}
\newreptheorem{fact}{Fact}

\usepackage[linesnumbered,vlined,ruled,commentsnumbered,nofillcomment]{algorithm2e}

\SetCommentSty{mycommfont}
\usepackage{enumitem}
\usepackage{import}
\usepackage[skins]{tcolorbox}
\usepackage{thmtools}
\usepackage{thm-restate}
\usepackage{breqn}
\usepackage{xspace}

\usepackage{todonotes}
\newcommand{\dan}[2][]{\todo[linecolor=blue,backgroundcolor=blue!25,bordercolor=blue,#1]{{\bf Dan:} #2}}
\newcommand{\giorgi}[2][]{\todo[linecolor=green,backgroundcolor=green!25,bordercolor=green,#1]{{\bf Giorgi:} #2}}

\usepackage{mathpartir}
\usepackage{etoolbox}
\usepackage{color}

\newcommand{\modl}{parameter\xspace}
\newcommand{\modls}{parameters\xspace}

\newcommand{\tg}{\widetilde{G}}
\newcommand{\fadd}{\texttt{fetch\&add}}
\newcommand{\norm}[1]{\left\|#1\right\|}
\renewcommand{\paragraph}[1]{\noindent\textbf{#1}}

\makeatletter
\DeclareRobustCommand{\qed}{%
	\ifmmode 
	\else \leavevmode\unskip\penalty9999 \hbox{}\nobreak\hfill
	\fi
	\quad\hbox{\qedsymbol}}
\newcommand{\proofsname}{Proof-Sketch}
\makeatother

\title{Elastic Consistency: A Practical Consistency Model for Distributed Stochastic Gradient Descent}
\begin{document}

\author{
	Giorgi Nadiradze \\
	IST Austria
	\and
	Ilia Markov \\
	IST Austria 
	\and
	Bapi Chatterjee \\ IST Austria
	\and
	Vyacheslav Kungurtsev \\
	Czech Technical University in Prague\\
	\and
	Dan Alistarh \\ 
	IST Austria
}
\date{}




\maketitle

\vspace{-3em}
\begin{abstract}
	\input{abstract}
\end{abstract}

\section{Introduction}
\input{introduction}

\section{Elastic Consistency}
\input{prelim}
\input{econsistency}

\section{Elastic Consistency and SGD Convergence}\label{sec:conv}
\input{convergence.tex}

\section{Distributed System Models and their Elastic Consistency Bounds}
\label{sec:ecbounds}
\subsection{Fault-Tolerant Message-Passing Systems}\label{sec:mpi}

\input{mpi.tex}

\section{Related Work and Discussion}
\input{relatedarxiv}
\bibliographystyle{plain}

\input{output.bbl}
\input{append.tex}

\end{document}

%% file: abstract.tex
One key 
element behind the progress of machine learning in recent years has been the ability to train machine learning models in large-scale distributed shared-memory and message-passing 
environments. 
Most of these models are trained  employing variants of stochastic gradient descent (SGD) based optimization. 
In this paper, we introduce a general consistency condition covering communication-reduced and asynchronous 
distributed SGD implementations. 
Our framework, called elastic 
consistency, decouples the system-specific aspects of the implementation from the SGD convergence requirements, 
giving a general way to obtain 
convergence bounds for a wide variety of distributed SGD methods used in practice. 
Elastic consistency can be used to re-derive or improve several previous convergence bounds in message-passing and shared-memory settings, but also to analyze new models and distribution schemes. 
In particular, we propose and analyze a new synchronization-avoiding scheme for distributed SGD, and show that it can be used to efficiently train deep convolutional models for image classification.


%% file: introduction.tex
Machine learning models can match or surpass humans on specialized tasks such as image classification~\cite{AlexNet,he2016deep}, speech recognition~\cite{seide2014sgd1bit}, or complex games~\cite{AlphaGo}. 
One key tool behind this progress is the \emph{stochastic gradient descent (SGD)} family of methods~\cite{robbins1951stochastic}, which are by and large the method of choice for training large-scale machine learning models. 
Essentially, SGD can serve to minimize a $d$-dimensional function $f: \R^d \to \R$, assumed to be differentiable. We commonly assume that we are given access to (possibly noisy) gradients of this function, denoted by $\tg$. 
Sequential SGD will start at a randomly chosen point $\vec{x}_0$, say $0^d$, and converge towards a minimum of the function by iterating the following procedure: 
\begin{equation}    
\label{eq:sgd}
 \vec{x}_{t + 1} = \vec{x}_{t} - {\alpha} \tg(\vec{x}_t)
\end{equation}
\noindent where $\vec{x}_t$ is the current estimate of the optimum, also called the \emph{parameter} and ${\alpha}$ is the \emph{learning rate}. If the function is convex, this procedure is known to converge to the minimum of the function~\cite{Bubeck}, whereas in the non-convex case, it will converge towards a point of zero gradient~\cite{ghadimi-lan}. 
In \emph{supervised learning},  $f$  is usually the total error of a given \modl $\vec{x}$  on a given dataset $\mathcal{D}$. For each sample $s$ in $\mathcal{D}$, the classification error is encoded via the loss $\ell(s, \vec{x})$. 
Training minimizes the function $f(\vec{x}) = \frac{1}{m} \sum_{s \in \mathcal{D}} \ell(s, \vec{x}),$
where $m$ is the size of the dataset, and the gradient at a randomly chosen datapoint $\tg$ is an unbiased estimator of $\nabla f$. 

Due to the size of datasets, it is common to \emph{distribute} the optimization process across multiple processors.  
A standard way of parallelizing SGD is to process a \emph{batch} of samples  in parallel, dividing the computation of gradient updates among  processors. Assume for simplicity that each processor is allotted one sample, whose corresponding gradient it computes  with respect to the current \modl $\vec{x}_t$. Processors then \emph{sum} their stochastic gradients, and update their local \modls by the resulting sum, leading to the following global iteration:
\begin{equation}    
\label{eq:parallel-sgd}
    \vec{x}_{t + 1} = \vec{x}_{t} - \frac{\alpha}{p} \sum_{i = 1}^p \tg^i(\vec{x}_t), 
\end{equation}
\noindent where $\tg^i$ is the stochastic gradient obtained at the processor $i$ at the given step, and $p$ is the batch size, equal to the number of processors. 
Since this sum is the same at every processor, this procedure yields the same, perfectly consistent, \modl at each processor at the end of every parallel iteration. 
The average $(1 / p) \sum_{i = 1}^p \tg^i(\vec{x}_t)$ is still a stochastic gradient, but with \emph{lower variance} than gradients at single samples, which can lead to better convergence~\cite{Bubeck}. Since samples are now processed in parallel,  the number of samples processed per second should in theory be multiplied by $p$. 

However, in practice, maintaining perfect consistency of the \modl $\vec{x}_t$ can negate the benefits of parallelization. 
Keeping the \modls perfectly consistent has a \emph{communication cost}: since the size of gradient updates is \emph{linear} in the size of the \modl, the resulting communication may easily become a system bottleneck for large-scale models, which can have millions of parameters~\cite{AlexNet,QSGD}. 
Consistency also induces a \emph{synchronization cost}, since processors need to synchronize in a barrier-like fashion upon each iteration update, which can occur every few milliseconds. 
For this reason, there have been several proposals for relaxing the consistency requirements
of SGD-like iterations, under various system constraints. These proposals can be broadly categorized as follows: 

\begin{itemize}[nolistsep, nosep, leftmargin=4mm]
    \item \textbf{Asynchronous Methods:} Such implementations~\cite{hogwild, lian2015asynchronous} allow  processors to forgo the barrier-like synchronization step performed at each iteration, or even across parameter components, and move forward with computation without waiting for potentially slow straggler processors. 

    \item \textbf{Communication Compression:} These methods aim to reduce the bandwidth cost of exchanging the gradients. This usually entails performing (possibly lossy) compression of the gradients before transmission, followed by efficient encoding and reduction/summation, and decoding on the receiver end. 
    This can be either via bit-width reduction (quantization), e.g.~\cite{seide2014sgd1bit,QSGD}, or via structured sparsification of the updates~\cite{lin2017deep,aji2017sparse, TopK, stich2019error}. 
\end{itemize}

Additional approaches exist, for instance to reduce the \emph{frequency} of communication via large-batch methods or local steps, e.g.~\cite{goyal2017accurate, lin2018don, stich2018local}. 
Another axis controls parameter maintenance: 
\emph{centralized} methods such as the \emph{parameter server}~\cite{PS} maintain the parameter at a single entity, whereas \emph{decentralized} methods~\cite{lian2017can} have each processor maintain their own version of the model. 

The question of providing convergence bounds for distributed optimization goes back to the foundational work of Bertsekas and Tsitsiklis~\cite{BT}, and has recently risen to prominence~\cite{dean2012large,ho2013more,chilimbi2014project,hogwild, ho2013more, desa2015hogwild, lian2015asynchronous,  chaturapruek2015asynchronous, leblond2016asaga}. 
However, many of these proofs are often specialized to the algorithm and models, and do not generalize to different settings. 
It is therefore natural to ask: are there generic conditions covering all natural consistency relaxations for SGD, under which one can prove convergence? 

\paragraph{Contribution.} 
In this paper, we introduce a convergence criterion for SGD-based optimization called \emph{elastic consistency}, which is independent of the system model, but can be specialized to cover various distributed system models consistency relaxations. 

In a nutshell, elastic consistency says that, for SGD to converge, it is sufficient that the distance  between 
the \emph{view} of the parameter perceived by a processor, with respect to which the gradient is taken, and the ``true’’ view of the system, 
corresponding to all the updates to the parameter generated up to that point  
by all processors, be uniformly bounded across iterations, and decreasing proportionally to the learning rate. 
Intuitively, in this case, the perturbed iterates do not stray ``too far'' from eachother, and can still globally converge. 
To our knowledge, elastic consistency is satisfied in most settings where asynchronous or communication-reduced methods have been 
analyzed so far, although proving this property for some systems is not always immediate. 

Elastic consistency provides a unified analysis framework for all the method types discussed above. 
Consequently, we are able to re-prove or improve convergence bounds for several methods, and to tackle new models and consistency relaxations.  
Our contributions are as follows: 
\begin{enumerate}[leftmargin=4mm]
	\item Under  standard smoothness assumptions on the loss, elastic 
	consistency is sufficient to guarantee convergence rates for inconsistent 
	SGD 
	iterations for both \emph{convex} and \emph{non-convex} problems. This condition is also 
\emph{necessary} for SGD convergence: we provide simple worst-case instances where SGD convergence is linear in the elastic consistency parameter. 
	
	\item Elastic consistency is satisfied by both asynchronous 
	message-passing and shared-memory models, centralized or decentralized, with or without faults, and by 
	communication-reduced methods. 
	This implies new convergence bounds for SGD in the 
	classic asynchronous and semi-synchronous message-passing 
	models~\cite{attiya2004distributed} and extends previous analyses
	for the shared-memory model~\cite{desa2015hogwild,alistarh2018convergence}. 
	
	\item This convergence condition inspires a new scheduling mechanism for parallel SGD, called \emph{elastic scheduling}, which controls communication-compression and asynchrony guided by our consistency condition in order to reduce synchronization overhead, while maintaining accuracy. 
	Its implementation provides non-trivial speedup upon BytePS~\cite{byteps}, the state-of-the-art scheduler for training deep neural networks.


\end{enumerate}

%% file: prelim.tex
\paragraph{Distributed Model and Adversarial Scheduling.} 
We consider distributed systems consisting of $p$ processors: $\{1, 2, \ldots, p\}$, some of which may 
be faulty, where communication happens either by message-passing, or via 
shared-memory. 
For simplicity, we will specify the system and fault models in the corresponding sections. 
We assume that the scheduling of steps (e.g. reads/writes in 
shared-memory, or message delivery in message-passing) is controlled by an \emph{oblivious} adversary. Thus,  
scheduling decisions are \emph{independent} of the randomness in the algorithm, and in particular of the data sampling. 
Practically, this implies that the 
conditioning on any random event involving previous SGD iterations $s<t$ does not impact choices made at iteration $t$. 

\paragraph{Distributed Optimization.} 
We assume that each of the $p$ processors is given access to  random samples coming from an unknown $d$-dimensional data distribution $\mathcal{D}$, and collaborates to jointly minimize $f: \mathcal{X} \rightarrow \R$ over the distribution $\mathcal{D}$, where $\mathcal{R}^d$ is a compact subset of $\R^d$. 
In practice, nodes optimize over a finite set of samples ${S} = \{S_1, S_2, \ldots, S_m\},$ and the function $f$ is defined as 
\begin{equation}
	\label{def:objective_function}
f(\vec{x}) = \frac{1}{m} \sum_{i = 1}^m \ell (S_i, \vec{x}) 
\end{equation}
where $\ell$ is the loss function at a sample $s$. The goal is to find  $\vec{x}^* \in \mathcal{R}^d$, which minimizes the expected loss over samples, defined as: $   \vec{x}^* = \text{ argmin}_{\vec{x}} f(\vec{x}) = \text{ argmin}_{\vec{x}} \E_{s \sim \mathcal{D}} [\ell (s, \vec{x})].$

\paragraph{Properties of Stochastic Gradients.} 
Let $\mathcal{A}$ be the 
$\sigma$-algebra over the space of random events -- arising from the randomness 
of stochastic gradients due to random sampling -- for all iterations $t$ and 
\modl vector $\vec{x}$. Let $\mathcal{F} = 
\left\{\mathcal{F}_t\right\}_{t\ge0}$ be a filtration of $\mathcal{A}$ with 
respect to the iterations $t$, wherein $\mathcal{F}_t$ pertains to all events 
until and including iteration $t$. 
We work under the following 
standard assumptions about stochastic gradients~\cite{desa2015hogwild}:

\begin{enumerate}[leftmargin=5.5mm]
	\item\textbf{Unbiasedness.} The i.i.d. stochastic gradients are unbiased estimators of the true gradient of the function $f$, that is:
	\begin{equation}
	\label{eqn:unbiased_estimator}
	\forall \vec{x} \in \R^d,~\E \left[\tilde{G}(\vec{x})\right] = \nabla f(\vec{x}).
	\end{equation}
	\item\textbf{Bounded Variance.} It is common to assume bounded variance of the  stochastic gradients: 
	\begin{equation}
	\label{eqn:variance_is_bounded_assumption_f}
	\forall \vec{x} \in \R^d,~\E\left[\|\tg\left(\vec{x}\right) - \nabla f(\vec{x}) \|^2\right] \leq \sigma^2.
	\end{equation}
	
	\item\textbf{Bounded Second Moment.} It is sometimes assumed that the second moment of the stochastic gradients over the sample space is bounded: 
	\begin{equation}
	\label{eqn:grad_is_bounded_assumption_f}
	\forall \vec{x} \in \R^d,~\E\left[\|\tg\left(\vec{x}\right)\|^2\right] \leq M^2.
	\end{equation}

	\noindent Elastic consistency \emph{does not}  require the second moment bound to ensure convergence---the variance bound is sufficient. However, in  system settings such as asynchronous shared-memory~\cite{desa2015hogwild}, this stronger assumption is common, and we will use it to bound the elastic consistency constant. 
	
\end{enumerate}

\paragraph{Properties of the Objective Function.} 
We will make use of the following standard definitions regarding the objective function $f$: $\forall~\vec{x},\vec{y} \in \mathbb{R}^d,$
\begin{enumerate}[leftmargin=5.5mm]
	\item\textbf{Smoothness.} The function $f: \mathbb{R}^d \rightarrow \mathbb{R}$ is smooth iff:
	\begin{equation}
\label{eqn:smooth_gradients_f}
\|\nabla f\left(\vec{x}\right) - \nabla f\left(\vec{y}\right)\| \leq L\|\vec{x}-\vec{y}\|~\text{for}~L>0.\\
\end{equation}
\item \textbf{Strong convexity.} Problems such as linear regression have a strongly convex objective:
\begin{equation}\label{eqn:strongconvex}
(\vec{x}-\vec{y})^T(\nabla f(\vec{x})-\nabla f(\vec{y})) \ge c\norm{x-y}^2~\text{for}~c>0.
\end{equation}
For such functions, the bound over second moment of stochastic gradients does not hold $\forall~\vec{x} \in \mathbb{R}^d$ \cite{NguyenNDRST18}. Therefore,  we restrict  $f: \mathcal{X} \rightarrow \mathbb{R}$ for a convex set $\mathcal{X} \subset \mathbb{R}^d$, such that $\forall x \in \mathcal{X}$, (\ref{eqn:grad_is_bounded_assumption_f}) is satisfied.
For simplicity, we omit the projection step onto $\mathcal{X}$, in the case when $\vec{x_t}$ does not belong to $\mathcal{X}$ for some iteration $t$.

\item \textbf{Lower bound for non-strongly-convex functions.} In many settings, such as training of neural networks, the objective function is not necessarily strongly-convex. In such ``non-strongly-convex'' settings, it is necessary to assume that $f$ is bounded from below:
\begin{equation}
\label{eqn:lower_bound_f}
\exists f^* \textnormal{ finite s.t. } \forall~\vec{x} \in \mathbb{R}^d,~f(\vec{x} )\ge f^*.
\end{equation}
\end{enumerate}

%% file: econsistency.tex
\subsection{Elastic Consistency Definition}

\paragraph{An Abstract Consistency Model.} 
We assume that we have $p$ processors, which share a \emph{\modl oracle} $\mathcal{O}$. 
In each iteration, each processor $i$ invokes this oracle, and receives a \emph{local view}  at of the \modl at step $t$, which we denote $\vec{v}_t^i$.
 The processor then uses this view of the \modl to generate a new update 
(stochastic gradient).  

The key question is how to express the consistency of the local views across 
processors. For this, we introduce an auxiliary variable $\vec{x}_t$, which we call 
the \emph{global parameter}.
Initially we have that $\vec{x}_0=\vec{v}_0^1=...=\vec{v}_0^p$.
Given $\vec{x}_t$, we consider two cases, depending on how data-parallel SGD is implemented.
The first is the \textbf{single steps} case, where the gradient generated locally by each processor is directly applied to the model. 
This is the case in completely asynchronous shared-memory~\cite{desa2015hogwild} or message-passing implementations~\cite{lian2015asynchronous}, and is modelled as: 
\begin{equation} \label{eqn:SGDsinglestep}
\vec{x}_{t+1}=\vec{x}_t-\alpha \tilde{G} (\vec{v}_t^i).
\end{equation}
The second is the \textbf{parallel steps} case, where processors' gradients are aggregated before they are applied, which is common in synchronous message-passing settings, e.g.~\cite{PS}.
Formally, we are given a set $I_t \subseteq \{1,2...,p\}$, such that $p/2 \le |I_t| \le p$.
Each processor $i \in I_t$  calculates a stochastic gradient based on its local view at step $t$, and we have that:
\begin{equation} \label{eqn:SGDmultiplesteps}
\vec{x}_{t+1}=\vec{x}_t-\frac{\alpha}{p} \sum_{i \in I_t} \tilde{G} (\vec{v}_t^i).
\end{equation}

With this in place, our consistency condition can be formulated as follows:

\begin{definition}[Elastic Consistency]
	\label{def:stochastic-consistency} 
	A distributed system provides \emph{elastic consistency} for the SGD iteration defined in~(\ref{eqn:SGDsinglestep}) and (\ref{eqn:SGDmultiplesteps}), if there exists a constant $B > 0$, independent of the iteration count $t$ but possibly dependent on the system definition, which bounds the expected norm difference between the true \modl $\vec{x}_t^i$ , and the view $\vec{v}_t^i$ returned by the oracle of processor $i$ at iteration $t$.
	Formally, 
	\begin{equation}
	\label{eqn:x-def}
	\mathbb{E}\left[ \| \vec{x}_t - \vec{v}_t^i\|^2\right] \leq \alpha^2 B^2,
	\end{equation}	
	where $B > 0$, $\alpha$ is the learning rate at iteration $t$, and the expectation is taken over the randomness in the algorithm. We call $B$ the \em{elastic consistency constant}.
\end{definition}

In Section~\ref{sec:ecbounds}, we show that virtually all known models and consistency relaxations satisfy this condition in some form (see Table~\ref{table:summary}). 
Elastic consistency gives very wide latitude to the \modl 
oracle about which exact values to return to the processor: the 
returned view can contain random or even adversarial noise, or updates 
may be delayed or missing, as long as their relative weight is bounded and independent of 
time.

%% file: convergence.tex
We now show that this notion of consistency implies non-trivial convergence guarantees for SGD for different types of objective functions, and that this notion is in some sense necessary for convergence. 
The complete proofs are available in the Appendix. 
\vspace{-0.7em}

\subsection{Elastic Consistency is Sufficient for SGD Convergence}
\vspace{-0.7em}

\paragraph{The Non-Convex Case.}
We begin with the more general case where the objective function is not necessarily convex. In this case, since convergence to a global minimum is not guaranteed for SGD, we will only require convergence to a point of  vanishing gradients, as is standard, e.g.~\cite{JiLiu}. Specifically, assuming elastic consistency, we prove the following theorems:

\begin{theorem} \label{thm:nonconvexsingle}
        Consider SGD iterations defined in 
	    (\ref{eqn:SGDsinglestep}) and satisfying elastic consistency bound (\ref{eqn:x-def}). For a smooth non-convex objective function $f$, whose minimum $x^*$ we are trying to find and the constant learning rate $\alpha = \frac{1}{\sqrt{T}}$, where $T\ge 36L^2$ is the number of iterations:
		
		\[
			\min_{t\in[T-1]}\mathbb{E}\|\nabla f(\vec{x}_t)\|^2 \le \frac{4(f(\vec{x}_0)-f(x^*))}{\sqrt{T}} +\frac{2B^2 L^2}{T} +\frac{6L\sigma^2}{\sqrt{T}}+\frac{6L^3B^2}{T\sqrt{T}}.			
		\]
\end{theorem}

\begin{theorem} \label{thm:nonconvexmultiple}
        Consider SGD iterations defined in 
	    (\ref{eqn:SGDmultiplesteps}) and satisfying elastic consistency bound (\ref{eqn:x-def}). For a smooth non-convex objective function $f$, whose minimum $x^*$ we are trying to find and the constant learning rate $\alpha = \frac{\sqrt{p}}{\sqrt{T}}$, where $T \ge 64L^2p$ is the number of iterations:
		
		\[
			\min_{t\in[T-1]}\mathbb{E}\|\nabla f(\vec{x}_t)\|^2 \le \frac{8(f(\vec{x}_0)-f(x^*))}{\sqrt{Tp}} +\frac{4B^2 L^2p}{T} +\frac{8L\sigma^2}{\sqrt{Tp}}+\frac{16L^3B^2p\sqrt{p}}{T\sqrt{T}}.			
		\]
\end{theorem}


\paragraph{The Strongly Convex Case.}
We can provide improved  guarantees under strong convexity: 

		

\begin{theorem} \label{thm:convexsingle}
        Consider SGD iterations defined in 
	    (\ref{eqn:SGDsinglestep})  and satisfying elastic consistency bound (\ref{eqn:x-def}). For a smooth, strongly convex function $f$, whose minimum $x^*$ we are trying to find and the constant learning rate $\alpha = \frac{2\log{T}}{cT}$, where $T \ge \frac{144L^2}{c^2}$ is a number of iterations:
	    \begin{equation*}
	      \E\|\vec{x}_{T}-x^*\|^2 \le   \frac{\|\vec{x}_0-x^*\|^2}{T}+\frac{16\log^2{T}L^2B^2}{c^4T^2}+\frac{12\sigma^2\log{T}}{T}+
        \frac{48\log^3{T}B^2L^2}{c^4T^3}.
	    \end{equation*}
	    
\end{theorem}

\begin{theorem} \label{thm:convexmultiple}
        Consider SGD iterations defined in (\ref{eqn:SGDmultiplesteps}) and satisfying elastic consistency bound (\ref{eqn:x-def}). For a smooth, strongly convex function $f$, whose minimum $x^*$ we are trying to find and the constant learning rate $\alpha = \frac{2(\log{T}+\log{p})}{cT}$, where $T \ge \frac{256L^2p}{c^2}$ is a number of iterations:
        \begin{align*}
        \E\|\vec{x}_{T}-x^*\|^2 \le \frac{\|\vec{x}_0-x^*\|^2}{Tp}&+\frac{16(\log{T}+\log{p})^2L^2B^2}{c^4T^2}+\frac{12\sigma^2(\log{T}+\log{p})}{Tp}\\&+
        \frac{48(\log{T}+\log{p})^3B^2L^2}{c^4T^3}.
        \end{align*}
\end{theorem}

\paragraph{Discussion.}
The parameter $B$ abstracts the distributed-system specific parameters to provide a clean derivation of the convergence theory.
In turn, depending on the system setting, $B$ might depend on the second-moment bound $M^2$ or variance bound $\sigma^2$, but also on system parameters such as the maximum delay $\tau$, on the number of failures $f$, or on the characteristics of the compression scheme. 
Specifically, assuming that the parameters are constant, the  convergence of the non-convex objectives is at a rate of $O(1/\sqrt{T})$ for SGD iterations defined in (\ref{eqn:SGDsinglestep}) and $O(1/\sqrt{Tp})$ for SGD iterations defined in (\ref{eqn:SGDmultiplesteps}). For a strongly-convex objective, we achieve rates of $\tilde{O}(1/T)$ and $\tilde{O}(1/(Tp)))$ for SGD iterations defined in (\ref{eqn:SGDsinglestep}) and (\ref{eqn:SGDmultiplesteps}) correspondingly (Here, $\tilde{O}$ hides $\log{T}$ and $\log{p}$ factors).

In Section~\ref{sec:exp} we exhibit natural distribution schemes for which the bound on $B$ does not require a second-moment bound on gradients. 
	
\vspace{-0.7em}
	
\subsection{Elastic Consistency is Necessary for SGD Convergence}	
We now show that elastic consistency can be directly linked to convergence; in particular, in the worst case, the convergence slow-down can be linear in $B^2$. 
The intuition is that  an adversarial oracle  can force the worker to always evaluate at a point that is $B$ away from the correct iterate, which can be shown to cause a linear slow-down even for quadratic functions. 

\begin{lemma}[Convergence Lower Bound]
\label{lem:lowerbound}
There exists a convex (quadratic) function $f$ and an adversarial oracle 
$\mathcal{O}$ s.t. $\mathbb{E}\left[\|\vec{x}_T-\vec{x}^*\|^2\right] \le \epsilon$ is achieved only after $T = \Omega\left(\frac{B^2}{\epsilon}\log{\left(\frac{1}{\epsilon}\right)}\right)$ iterations.
\end{lemma}
\vspace{-0.7em}

%% file: mpi.tex
We consider a
message-passing (MP) system of $p$ nodes $\mathcal{P} = \{1, 2, 
\ldots, p\}$ executing SGD iterations, which are connected to each other by point-to-point links. 
To simplify the description, we will focus on the case where the system is \emph{decentralized}: 
in this case, each node $i$ acts both as a worker (generating gradients) and as a parameter server (PS) (maintaining a local parameter copy). 
(The only difference is that, in the \emph{centralized} case, a single designated node would maintain a global parameter copy.) 
Without loss of generality, all nodes start with the same parameter $\vec{v}^{i}_{0} = \vec{0}$.  
The system proceeds in global iterations, indexed by $t$. In each iteration, each worker generates a stochastic gradient based on its current model copy $\vec{v}^{i}_{t}$, and broadcasts it to all other nodes. 

\paragraph{Consistency Relaxations:} 
Consider node $i$, and recall that it acts
as a parameter server, in addition to being a worker itself. 
Let 
$\mathcal{L}_t^i \subseteq \{1,\ldots,p\}$ denote the set of nodes from which $i$ receives stochastic gradients at iteration $t$. 
By convention, $i \in \mathcal{L}_t^i$.
In the \emph{synchronous failure-free} case, 
all nodes would receive exactly the same set of messages, and the execution would be equivalent to a sequential batch-SGD execution. 
In a real system, not all nodes may have the same view of each round, due to failures or asynchrony. Specifically, we consider the following distinct consistency relaxations: 

\begin{enumerate}[leftmargin=6.5mm,label=(\alph*)]
	\item\textbf{Crash faults.} A node $i \in \mathcal{P}$ may \emph{crash} during computation or while sending messages. 
	In this case, the node will remain inactive for the rest of the execution. 
	Importantly, node $i$'s crash during broadcasting may cause other nodes to have different views at the iteration, as some of them may receive $i$'s message while others may not, resulting in 
	inconsistent updates of \modl $\vec{x}$ across the nodes.
	We will assume that  $f \le p/2$ nodes may crash in the MP system.
	
	\item\textbf{Message-omission failures.} 
	In practical systems, each node could implement iteration $t$ by waiting until an interval
	$\Upsilon_{max}^t$ in terms of \emph{clock time} has elapsed: if the message from peer $j$ is not received in time, node  $i$ moves on, updating its local \modl $\vec{x}^i_t$ only w.r.t. received messages. 
	However, node $i$ will include $j$'s message into its view if received in a later iteration, although some messages may be permanently delayed.
	We assume a parameter $f$ which upper bounds the number of messages omitted at any point during the execution. 
	We note that this model is stronger than the Crash-fault model considered above, as we can simulate a node's failure by discarding all its messages after the crash. 
	
	\item\textbf{Asynchrony.} 
	The two above models assume that nodes proceed synchronously, in \emph{lock-step}, although they may have inconsistent views due to node or message failures. 
	An alternative relaxation~\cite{lian2015asynchronous} is if nodes proceed \emph{asynchronously}, i.e. may be in different iterations at the same point in time. 
	Specifically, in this case, it is assumed that there exists a maximum delay $\tau_\max$ such that each message/gradient can be delayed by at most $\tau_\max$ iterations from the iteration when it was generated. 
	
	\item\textbf{Communication-Compression.}
	Another way of reducing the distribution cost of SGD has been to compress the stochastic gradients communicated at each round. 
	In this context, sparsification with memory~\cite{strom2015scalable, seide2014sgd1bit, aji2017sparse, TopK, KarimireddyRSJ19} has proven to be particularly effective. 
	This process can be modelled as follows. 
	We assume that each node maintains a local version of the parameter $\vec{v}^{i}_t$, and an \emph{error/memory} vector $\vec{\epsilon}^{i}_{t}$, initially $\vec{0}$. 
	In each iteration, each node computes a new gradient $\tg\left(\vec{v}^{i}_t\right)$ based on its local parameter. 
	It then adds the current error vector $\vec{\epsilon}^{i}_{t}$ to the gradient, to obtain its full proposed update $\vec{w}^{i}_{t}$. 
	However, before transmitting this update, it compresses it by using a (lossy) compression function $Q$. 
	The compressed update $Q(\vec{w}^{i}_{t})$ is then transmitted, and the error vector is updated to $\vec{\epsilon}^i_{t+1}\gets\vec{w}^{i}_{t}-Q(\vec{w}^{i}_{t})$. 
	Our analysis will only require that $Q$ satisfies 
	$\|Q(\vec{w})-w\|^2 \le 
\gamma\|\vec{w}\|^2,~\text{$\forall \vec{w} \in \R^d$, for some  $1 > \gamma \ge 0$}.$
	All memory-based techniques satisfy this, for various definitions of $Q$ and $\gamma$. 
	(We provide examples in the additional material.)

\end{enumerate}

We note that the above discussion  considered these methods independently. 
However, we do note that our method does allow for these relaxations to be combined---for instance, one can analyze an asynchronous, fault-tolerant method with communication compression. 
\vspace{-0.7em}

\subsection{Asynchronous Shared-Memory Systems} 
\vspace{-0.7em}

We consider a system with $p$ processors (or threads) $\mathcal{P} = \{1, 2, \ldots, p\}$, which can communicate through shared memory. 
Specifically, we assume that the \modl vector $\vec{w} \in \mathbb{R}^d$ is shared by the processors, and is split into $d$ components, one per dimension. 
Processors can atomically read a component via a \texttt{read} operation, and update it via the atomic  \fadd~(\texttt{faa}) operation, which reads the current value of the component and updates it in place, in a single atomic step. 
In each iteration $t$, each processor first obtains a local view $\vec{v}^{i}_t$ of the \modl by scanning through the shared parameter $\vec{w}$ component-wise. 
It then generates a stochastic gradient $\tg\left(\vec{v}^{i}_t\right)$ based on this view, and proceeds to update $\vec{x}$ via  \texttt{faa} on each component, in order. (We refer the reader to the Appendix for a full description, including pseudocode.)

\paragraph{Consistency Relaxation.} Ideally, threads would proceed in lock step, first obtaining perfect, identical snapshots of $\vec{w}$, calculating gradients in terms of this identical parameter, and then summing the gradients before proceeding to the next iteration. 
However, in practice, threads are \emph{asynchronous}, and proceed at arbitrary speeds. This causes their snapshots to be inconsistent, as they might contain some partial concurrent updates, but not others. The challenge is to prove SGD convergence in this case. 
It is common~\cite{hogwild, desa2015hogwild} to assume a bound $\tau_\max$ on the maximum delay between the time (iteration) when an individual update was generated, and the iteration when it has been applied, and becomes visible to all processors. 
In this case, the auxiliary variable $\vec{x}_t$ used by elastic consistency will correspond to the sum of first $t$ stochastic gradients, ordered by the time when the atomic \texttt{faa} over the first index of $\vec{w}$ was performed.

\vspace{-0.7em}
\subsection{Elastic Consistency Bounds for Distributed Systems}
\vspace{-0.7em}

Given these definitions, we can now state the elastic consistency bounds for the different types of distributed systems and consistency relaxations. Please see the Appendix for detailed derivations.  

\begin{table}[ht]
\begin{center}
{ \footnotesize
\begin{tabular}{|>{\centering\arraybackslash}p{18mm}|>{\centering\arraybackslash}p{72mm}|>{\centering\arraybackslash}p{22mm}|>{\centering\arraybackslash}p{18mm}|}

\hline
\textbf{System}		& \textbf{Consistency Relaxation} &  Bound $B$ & 
\textbf{Novelty}		\\
\hline \hline
Shared-Memory			& $\tau_{max}$-Bounded Asynchrony  & $\sqrt{d} \tau_{max} M$ & Extends~\cite{ desa2015hogwild, alistarh2018convergence}\\
\hline
Message-Passing		& $\tau_{max}$-Bounded Asynchrony & $\frac{ (p-1)\tau_{max}M}{p}$ & Reproves~\cite{lian2015asynchronous} \\
\hline
Message-Passing				& $\tau_{max}$-Bounded Asynchrony  & $O(\frac{(p-1)\tau_{max}\sigma}{p})$ & New\\
 \hline
Message-Passing		& \emph{Distributed} Communication-Compression ~~~~~~~~~~~~~~~~~~~~~~~~~~~~ with Error Feedback	&	$\sqrt{\frac{(2-\gamma)\gamma}{(1-\gamma)^3}}M$ & Improves~\cite{TopK, StichCJ18, KarimireddyRSJ19} \\
\hline
Message-Passing				& Synchronous, $f$ {Crash} or {Message-drop} Faults  & ${Mf}/{p}$ & {New} \\ \hline
Message-Passing				& Synchronous, $f$ {Crash} or {Message-drop} Faults  & $O(\sigma f/p)$ & {New} \\ \hline
Message-Passing				& Variance bounded \emph{Elastic Scheduler}
&  $O(\sigma)$ & 
{New} \\
\hline

\end{tabular}
}
\end{center}
\caption{Summary of elastic consistency bounds.}
\label{table:summary}
\end{table}

\paragraph{Implications.} Plugging the asynchronous shared-memory bound into Theorems~\ref{thm:nonconvexsingle} and~\ref{thm:convexsingle} implies convergence bounds in the smooth non-convex case, extending~\cite{desa2015hogwild, alistarh2018convergence}, which focus on the convex case, whereas the asynchronous message-passing bound implies similar bounds  to the best known in the non-convex case for this model~\cite{lian2015asynchronous}.  
For synchronous message-passing with communication-compression, our framework implies the first general bounds for the \emph{parallel, multi-node} case: references~\cite{StichCJ18, KarimireddyRSJ19} derive tight rates for such methods, but in the \emph{sequential} case, where there is a single node which applies the compressed gradient onto its model, whereas~\cite{TopK} considers the multi-node case, but requires an additional analytic assumption. 
Please see Section~\ref{sec:related} for additional discussion. 
In the crash-prone case, elastic consistency implies  convergence bounds for crash or message-omission faults. (Although the formula is the same, the failure definition and the derivations are different.) 
The framework also allows to combine consistency relaxations, i.e. consider communication-compression or asynchronous models with crashes. 

One relative weakness of the above results is that the bounds depend on the gradient second-moment bound. This is not due to elastic consistency itself, but due to the fact that we needed a bound on $M$ to bound the elastic consistency constant for these systems, which is consistent with previous work, e.g.~~\cite{desa2015hogwild, alistarh2018convergence, lian2015asynchronous}. 
Next, we remove this limitation by slightly altering the algorithms. 

\vspace{-1em}

\section{Practical Application: Elastic Scheduling} 
\vspace{-1em}

\label{sec:exp}

So far, we have used elastic consistency to derive bounds for existing models and methods. 
Next, we ask whether it can inspire \emph{new} distribution schemes. 
Our target application will be communication scheduling in the context of training deep neural networks (DNNs). 
More precisely, when performing distributed training of DNNs via back-propagation, it is common to schedule parts of the communication in parallel with the computation. For instance, assuming we train a three-layer network $A \rightarrow B \rightarrow C$, the gradient of the last layer $C$ will be ``ready'' to sync before layers $A, B$, and can be transmitted earlier. Several recent papers, e.g.~\cite{P3, byteps} investigate scheduling mechanisms for leveraging this type of communication-computation overlap. 
A common feature of these schedulers~\cite{P3, byteps} is  ensuring \emph{perfect consistency} at each processor: communication can be reordered w.r.t. computation only if it does not deviate from the sequential execution. 

Our analysis suggests that consistency can be relaxed as long as the consistency bound is \emph{small}. 
Specifically, we will allow  processes to start their next forward pass \emph{before} all layers are synchronized, as long as enough gradient norm has been received to ensure a small elastic consistency constant.  

\paragraph{The Elastic Scheduler.} 
We will present two schedulers, a \emph{norm-bounded} and a \emph{variance-bounded} one. 
Given $0 \leq \beta \leq 1$, the $\beta-$norm-bounded algorithm at a fixed processor $i$ is as follows. 
Assume that the processor is at iteration $t$, and has completed its backward pass over the network, obtaining a local gradient $\tg (\vec{v}_t^i)$ (computed over the local view $\vec{v}^t_i$).
Normally, the processor would need to wait for all parameters to be synchronized, i.e. to receive all the other processors' gradients.  
However, the \emph{elastic scheduling rule} will allow the processor to start its next forward pass \emph{before} this point, on the inconsistent view,
as long as the norm of the received update is at least a $\beta$-fraction of \emph{its own gradient} at the step. 
In this case, the processor speculatively goes forward with its forward-backward step. For both versions, the processor cannot speculate ahead by more than $1$ step. 
In this case, the elastic consistency constant $B$ is upper bounded by $O(M)$. 

The \emph{variance-bounded} version is slightly cleverer: if a processor finishes its forward-backward pass ``early'' and does not receive all the other processors' gradients within a small timeout, it will proceed to replace the missing gradients \emph{with its own}, and speculatively performs the forward-backward pass based on this inconsistent view. 
Critically, the processor will ``correct'' the gradient step retroactively, once it has received the full gradient in the next step. 
The consistency bound in this case becomes $3\sigma$. The proof is provided in Appendix (See section \ref{sec:elasticscheduling}).

More generally, variance-bounded scheduler also inspires a way of improving the elastic consistency bounds for crash and message-drop faults and asynchrony with delay $\tau_{max}$ : instead of proceeding without the dropped messages, each node can \emph{replace the corresponding missing gradient with its own}. This will allow
us to replace the second moment bound $M$ with variance bound $O(\sigma)$.


\begin{figure}[ht]
    \centering
    \begin{minipage}[b]{0.45\linewidth}
        \centering
        \includegraphics[width=7cm]{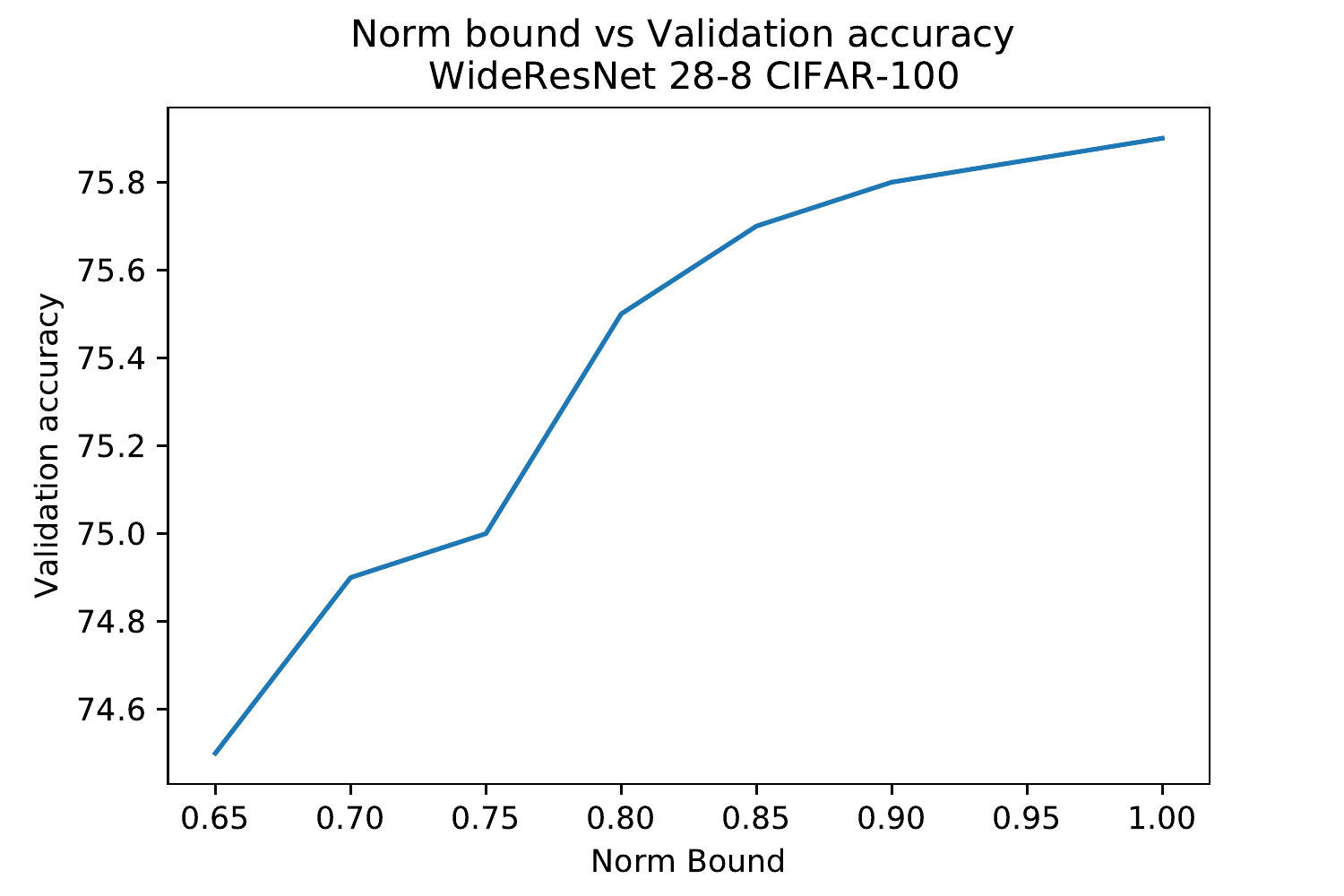}
        \label{fig:reg-variance}
    \end{minipage}
    \quad
    \begin{minipage}[b]{0.45\linewidth}
        \centering
        \includegraphics[width=7cm]{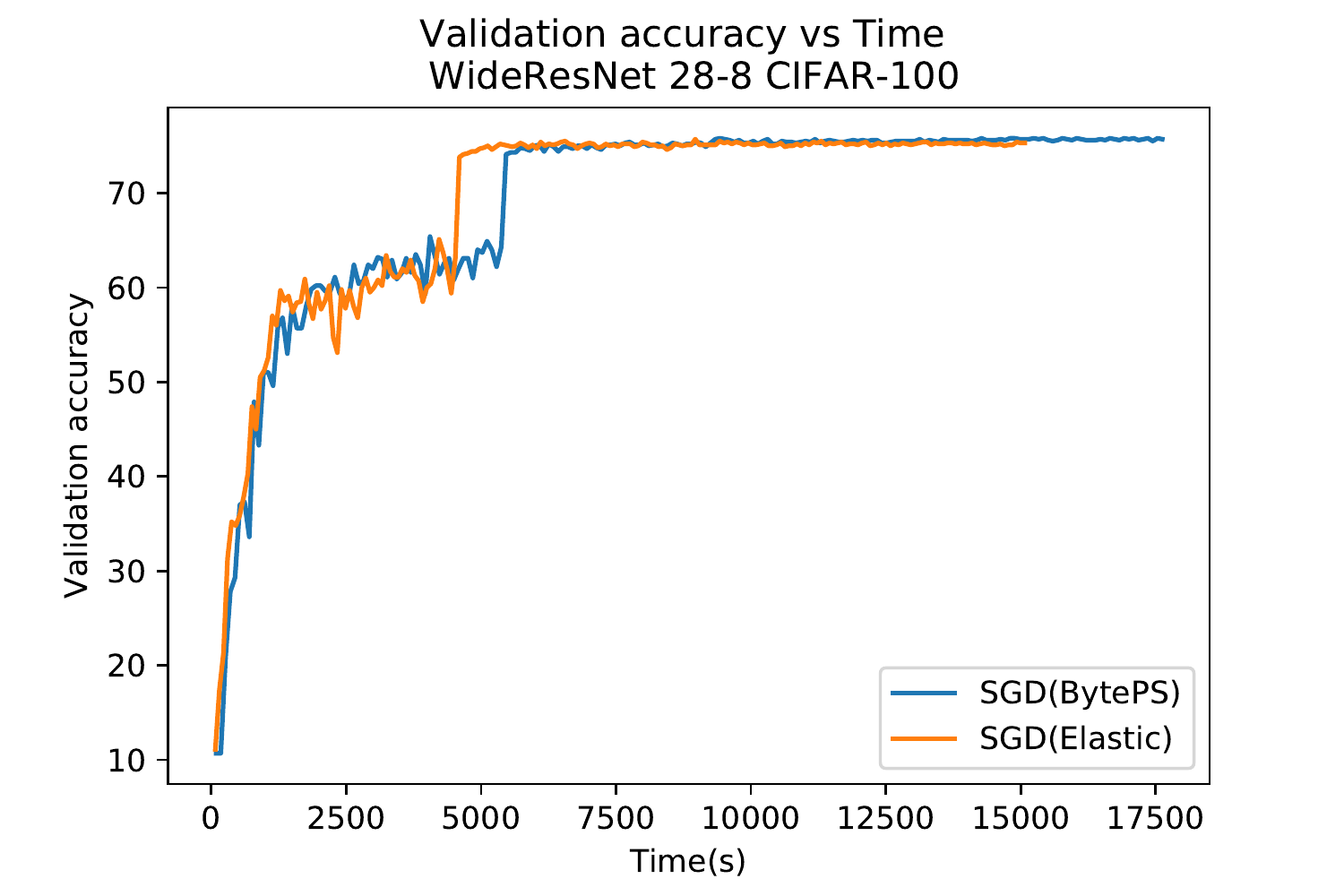}
        \label{fig:reg-conv}
    \end{minipage}
    \caption{Elastic bound-v-accuracy (\textbf{left}) and accuracy-v-time (\textbf{right}) for WRN28x8 on CIFAR-100.}
    \label{fig:exp}
\end{figure}

\paragraph{Implementation.} 
We implement the elastic scheduler on top of the Horovod distributed training framework~\cite{sergeev2018horovod}, which allows us to interface with both Pytorch~\cite{PyTorch} and Tensorflow~\cite{TF}. 
Our implementation is  decentralized---as opposed to BytePS~\cite{byteps}, which has a Parameter Server architecture---but  the performance of our perfectly-consistent implementation is identical to the original BytePS. 
We conduct our experiments in Pytorch, testing convergence and speedup for residual networks~\cite{he2016deep} applied to image classification tasks on the CIFAR-10/100 datasets~\cite{krizhevsky2009learning}. 
Hyperparameter values are standard, given in the Appendix. Experiments are performed on two AWS EC2 P3.2xlarge instances, each with a V100 GPU, and averaged over 3 trials.  

\paragraph{Experiments.} 
We first examine the impact of the elastic consistency bound on accuracy. 
For this, we execute the norm-bounded variant with different values of $\beta \in [0, 1]$, and examine the top validation accuracy.  
Figure~\ref{fig:exp} (left) shows the strong  correlation between these two measures for WideResNet28x8~\cite{zagoruyko2016wide} on CIFAR-100, confirming our analysis. 
Next, we examine the potential for speedup of the elastic scheduler. 
As frequent re-computation of the L2 norm is expensive, we implement a relaxed variant which tracks the ratio of parameters received (L0 norm). 
The results for $\beta = 0.8$ are given in Figure~\ref{fig:exp} (right), showing a $\sim 20\%$ speedup versus the (highly performant) baseline implementation of BytePS, without accuracy loss. 
The Appendix contains a full experimental report, with the variance-bounded version and additional models and datasets.

%% file: relatedarxiv.tex

\label{sec:related}
Distributed machine learning has recently gained significant practical adoption,  e.g.~\cite{dean2012large,ho2013more,chilimbi2014project, zhang2015deep, xing2015petuum, P3, byteps}. 
Consequently, there has been significant work on introducing and analyzing distributed relaxations of SGD~\cite{hogwild, ho2013more, desa2015hogwild, lian2015asynchronous,  chaturapruek2015asynchronous, leblond2016asaga,alistarh2018convergence,  wang2018cooperative, woodworth2018graph, KarimireddyRSJ19, stich2019error, lu2020mixml}. Due to space constraints, we cover in detail only work that is technically close to ours. 

Specifically, De Sa et al. \cite{desa2015hogwild} were the first to consider a unified 
analysis framework for asynchonous and communication-compressed iterations.
Relative to it, our framework improves in three respects: (i) it does not require 
stringent gradient sparsity assumptions; 
(ii) it is also able to analyze the case where the updates are not unbiased 
estimators of the gradient, which allows extensions to error-feedback communication-reduction; 
and (3) it also tackles convergence for general non-convex 
objectives. 
Reference~\cite{lian2015asynchronous} presented the first general analysis of asynchronous non-convex SGD , without  communication-reduction. 
Qiao et al.~\cite{DBLP:conf/icml/QiaoAZX19} model 
 asynchrony and communication reduction as
perturbations of the SGD iteration, and introduce a metric called ``rework cost,'' which can be subsumed into the elastic consistency bound.

Karimireddy et al.~\cite{KarimireddyRSJ19}
 analyze communication-compression with error feedback, and present a general notion of $\delta$-compressor to model communication-reduced consistency relaxations; later, the framework was extended to include \emph{asynchronous} iterations~\cite{stich2019error}. 
 Every method satisfying the $\delta$-compressor property is elastically-consistent, although the converse is not true. 
 Relative to this work, our framework generalizes in one important practical aspect, as it allows the analysis in \emph{distributed} settings:~\cite{KarimireddyRSJ19, stich2019error} assume that the iterations are performed at a single processor, which may compress gradients or view inconsistent information only with respect to \emph{its own} earlier iterations. 
 This extension is non-trivial; tackling this more realistic setting previously required additional analytic assumptions~\cite{TopK}. 
 
We have proposed a new and fairly general framework for analyzing inconsistent SGD iterations. 
Its main advantages are \emph{generality}---as we have shown, elastic consistency is satisfied by virtually all known distributed models and by many consistency relaxations---and \emph{simplicity}: as illustrated by stochastic scheduling, the method can reduce a complex analysis to simply bounding the elastic consistency constant. 
Inspired by this technical condition, we introduce new and efficient scheduling mechanisms. 
More generally, we believe that elastic consistency  could inspire new distributed variants, and be used to derive convergence in a streamlined, more modular manner. 
One key line of extension which we plan to pursue in future work is to study whether elastic consistency can be extended to other first-order distributed optimization methods, such as stochastic coordinate descent, or zeroth- or second-order optimization methods.

%% file: append.tex
\appendix
\tableofcontents

\section{Detailed Proofs}
\subsection{Complete Proof of Convergence in the Non-Convex Case}\label{proof:full-non-con}

\begin{lemma} \label{lem:singlestepnonconvex}
Consider SGD iterations defined in 
	(\ref{eqn:SGDsinglestep}) and satisfying the elastic consistency bound (\ref{eqn:x-def}). For a smooth non-convex objective function $f$ and the constant learning rate $\alpha \le \frac{1}{6L}$.
	We have that:
	\begin{equation}
	    \E[f(\vec{x}_{t+1})] \le  \E[f(\vec{x}_t)]-\frac{\alpha}{4} \E\|\nabla f(\vec{x}_t)\|^2+\frac{\alpha^3B^2 L^2}{2} +\frac{3L\alpha^2\sigma^2}{2}+\frac{3L^3\alpha^4B^2}{2}.  
	\end{equation}
\end{lemma}
\begin{proof}
We condition on $\vec{x}_t$ and $\vec{v}_t^i$ and calculate expectation with respect to the randomness of stochastic gradient (We call this $\E_s$).
By descent lemma we get that :
\begin{align*}
\E_s[f(\vec{x}_{t+1})] &\le f(\vec{x}_t)-\alpha \E_s \langle \tilde G(\vec{v}_t^i), \nabla f(\vec{x}_t)  \rangle+\frac{L^2\alpha^2}{2} \E_s \|\tilde G(\vec{v}_t^i)\|^2 \\& \overset{(\ref{eqn:unbiased_estimator})}{=} f(\vec{x}_t)-\alpha \|\nabla f(\vec{x}_t)\|^2+\alpha \langle \nabla f(\vec{x}_t) -  \nabla f(\vec{v}_t^i), \nabla f(\vec{x}_t) \rangle \\&\quad\quad\quad\quad+\frac{L\alpha^2}{2} \E_s\|\tilde G(v_t^i)- \nabla f(\vec{v}_t^i)+\nabla f(\vec{v}_t^i)-\nabla f(\vec{x}_t)+\nabla f(\vec{x}_t)\|^2 \\ &\overset{Cauchy-Schwarz}{\le}
f(\vec{x}_t)-\alpha \|\nabla f(\vec{x}_t)\|^2+\alpha \langle \nabla f(\vec{x}_t) -  \nabla f(\vec{v}_t^i), \nabla f(\vec{x}_t) \rangle \\&\quad\quad\quad\quad+\frac{3L\alpha^2}{2}\Big(
\E_s\|\tilde G(\vec{v}_t^i)- \nabla f(\vec{v}_t^i)\|^2+\|\nabla f(\vec{v}_t^i)-\nabla f(\vec{x}_t)\|^2+\|\nabla f(\vec{x}_t)\|^2\Big) \\ &\overset{Young}{\le}
f(\vec{x}_t)-\alpha \|\nabla f(\vec{x}_t)\|^2+\frac{\alpha}{2} \|\nabla f(\vec{x}_t) -  \nabla f(\vec{v}_t^i)\|^2 +\frac{\alpha}{2} \|\nabla f(\vec{x}_t)\|^2 \\&\quad\quad\quad\quad+\frac{3L\alpha^2}{2}\Big(
\E_s\|\tilde G(\vec{v}_t^i)- \nabla f(\vec{v}_t^i)\|^2+\|\nabla f(\vec{v}_t^i)-\nabla f(\vec{x}_t)\|^2+\|\nabla f(\vec{x}_t)\|^2\Big) \\&\overset{(\ref{eqn:variance_is_bounded_assumption_f})}{\le} 
f(\vec{x}_t)-\frac{\alpha}{2} \|\nabla f(\vec{x}_t)\|^2+\frac{\alpha}{2} \|\nabla f(\vec{x}_t) -  \nabla f(\vec{v}_t^i)\|^2 \\&\quad\quad\quad\quad+\frac{3L\alpha^2}{2}\Big(
\sigma^2+\|\nabla f(\vec{v}_t^i)-\nabla f(\vec{x}_t)\|^2+\|\nabla f(\vec{x}_t)\|^2\Big)
\\&\overset{(\ref{eqn:smooth_gradients_f})}{\le} f(\vec{x}_t)-\frac{\alpha}{2} \|\nabla f(\vec{x}_t)\|^2+\frac{\alpha L^2}{2} \|\vec{x}_t -  \vec{v}_t^i\|^2 +\frac{3L\alpha^2}{2}\Big(
\sigma^2+L^2\|\vec{v}_t^i-\vec{x}_t\|^2+\|\nabla f(\vec{x}_t)\|^2\Big)
\end{align*}
Next, we use elastic consistency bound (\ref{eqn:x-def}) and the fact that $\alpha \le \frac{1}{6L}$ in the above inequality:
\begin{align*}
    \E[f(\vec{x}_{t+1})]&=\E[\E[f(\vec{x}_{t+1})|\vec{x}_t,\vec{v}_t^i]] \\&\le 
    \E[f(\vec{x}_t)]-\frac{\alpha}{4} \E\|\nabla f(\vec{x}_t)\|^2+\frac{\alpha L^2}{2} \E\|\vec{x}_t -  \vec{v}_t^i\|^2 +\frac{3L\alpha^2}{2}\Big(
\sigma^2+L^2\E\|\vec{v}_t^i-\vec{x}_t\|^2\Big) \\ &\le
\E[f(\vec{x}_t)]-\frac{\alpha}{4} \E\|\nabla f(\vec{x}_t)\|^2+\frac{\alpha^3B^2 L^2}{2} +\frac{3L\alpha^2\sigma^2}{2}+\frac{3L^3\alpha^4B^2}{2}.
\end{align*}
\end{proof}

\begin{reptheorem} {thm:nonconvexsingle}
        Consider SGD iterations defined in 
	    (\ref{eqn:SGDsinglestep}) and satisfying elastic consistency bound (\ref{eqn:x-def}). For a smooth non-convex objective function $f$, whose minimum $x^*$ we are trying to find and the constant learning rate $\alpha = \frac{1}{\sqrt{T}}$, where $T\ge 36L^2$ is the number of iterations:
		
		\[
			\min_{t\in[T-1]}\mathbb{E}\|\nabla f(\vec{x}_t)\|^2 \le \frac{4(f(\vec{x}_0)-f(x^*))}{\sqrt{T}} +\frac{2B^2 L^2}{T} +\frac{6L\sigma^2}{\sqrt{T}}+\frac{6L^3B^2}{T\sqrt{T}}.			
		\]
\end{reptheorem}
\begin{proof}
   Observe that since $\alpha = 1/\sqrt{T} \le 1/6L$, by Lemma \ref{lem:singlestepnonconvex} 
   we have that :

   \begin{equation*}
    \E[f(\vec{x}_{t+1})] \le \E[f(\vec{x}_t)]-\frac{\alpha}{4} \E\|\nabla f(\vec{x}_t)\|^2+\frac{\alpha^3B^2 L^2}{2} +\frac{3L\alpha^2\sigma^2}{2}+\frac{3L^3\alpha^4B^2}{2}.
    \end{equation*}
    By summing the above inequality for $t=0,1,..,T-1$ we get that:
    \begin{equation*}
    \sum_{t=0}^{T-1} \E[f(\vec{x}_{t+1})] \le \sum_{t=0}^{T-1} \Big ( \E[f(\vec{x}_t)]-\frac{\alpha}{4} \E\|\nabla f(\vec{x}_t)\|^2+\frac{\alpha^3B^2 L^2}{2} +\frac{3L\alpha^2\sigma^2}{2}+\frac{3L^3\alpha^4B^2}{2} \Big).
    \end{equation*}
    Which can be rewritten as:
    \begin{equation*}
    \sum_{t=0}^{T-1} \frac{\alpha}{4} \E\|\nabla f(\vec{x}_t)\|^2 \le f(\vec{x}_0)-\E[f(\vec{x}_{T})] +\frac{\alpha^3B^2 L^2T}{2} +\frac{3L\alpha^2\sigma^2T}{2}+\frac{3L^3\alpha^4B^2T}{2}.
    \end{equation*}
	Next, we divide by $\frac{\alpha T}{4}$ and use the fact that $\E[f(\vec{x}_{T})]\ge f(x^*)$:
	\begin{equation*}
    \sum_{t=0}^{T-1} \frac{1}{T} \E\|\nabla f(\vec{x}_t)\|^2 \le \frac{4(f(\vec{x}_0)-f(x^*))}{\alpha T} +2\alpha^2B^2 L^2 +6L\alpha\sigma^2+6 L^3\alpha^3B^2.
    \end{equation*}
    
    By plugging in the value of $\alpha$ we get:
    \begin{align*}
			\min_{t\in[T-1]}\mathbb{E}\|\nabla f(\vec{x}_t)\|^2 &\le \sum_{t=0}^{T-1} \frac{1}{T} \E\|\nabla f(\vec{x}_t)\|^2 \le \frac{4(f(\vec{x}_0)-f(x^*))}{\alpha T} +2\alpha^2B^2 L^2 +6L\alpha\sigma^2+6L^3\alpha^3B^2 \\ &\le \frac{4(f(\vec{x}_0)-f(x^*))}{\sqrt{T}} +\frac{2B^2 L^2}{T} +\frac{6L\sigma^2}{\sqrt{T}}+\frac{6L^3B^2}{T\sqrt{T}}.
    \end{align*}

\end{proof}

\begin{lemma} \label{lem:multiplestepsnonconvex}
Consider SGD iterations defined in 
	(\ref{eqn:SGDmultiplesteps}) and satisfying elastic consistency bound (\ref{eqn:x-def}). For a smooth non-convex objective function $f$ and the constant learning rate $\alpha \le \frac{1}{8L}$.
	We have that:
	\begin{equation}
	    \E[f(\vec{x}_{t+1})] \le \E[f(\vec{x}_t)]-\frac{\alpha}{8} \E\|\nabla f(\vec{x}_t)\|^2+\frac{\alpha^3B^2 L^2}{2} +\frac{L\alpha^2\sigma^2}{p}+2L^3\alpha^4B^2.  
	\end{equation}
\end{lemma}
\begin{proof}
We condition on $\vec{x}_t$ and $\{\vec{v}_t^i|i \in I_t\}$ and calculate expectation with respect to the randomness of stochastic gradient(We call this $\E_s$).
By descent lemma we get that :
\begin{align*}
\E_s[f(\vec{x}_{t+1})] &\le f(\vec{x}_t)-\sum_{i \in I_t} \frac{\alpha}{p} \E_s \langle  \tilde G(\vec{v}_t^i), \nabla f(\vec{x}_t)  \rangle+\frac{L^2\alpha^2}{2p^2} \E_s \|\sum_{i\in I_t}\tilde G(\vec{v}_t^i)\|^2 \\& \overset{(\ref{eqn:unbiased_estimator})}{=} f(\vec{x}_t)-\frac{\alpha|I_t|}{p} \|\nabla f(\vec{x}_t)\|^2+\frac{\alpha}{p} \sum_{i \in I_t} \langle \nabla f(\vec{x}_t) -  \nabla f(\vec{v}_t^i), \nabla f(\vec{x}_t) \rangle \\&\quad\quad\quad\quad+\frac{L\alpha^2}{2p^2} \E_s\Bigg\|\sum_{i \in I_t} \Big(\tilde G(v_t^i)- \nabla f(\vec{v}_t^i)+\nabla f(\vec{v}_t^i)-\nabla f(\vec{x}_t)+\nabla f(\vec{x}_t)\Big)\Bigg\|^2 \\ &\overset{Cauchy-Schwarz}{\le}
f(\vec{x}_t)-\frac{\alpha|I_t|}{p} \|\nabla f(\vec{x}_t)\|^2+\frac{\alpha}{p} \sum_{i \in I_t} \langle \nabla f(\vec{x}_t) -  \nabla f(\vec{v}_t^i), \nabla f(\vec{x}_t) \rangle \\&\quad\quad\quad\quad+\frac{L\alpha^2}{p^2} \E_s\Bigg\|\sum_{i \in I_t} \Big(\tilde G(v_t^i)- \nabla f(\vec{v}_t^i)\Big)\Bigg\|^2 
\\&\quad\quad\quad\quad+\frac{L\alpha^2}{p^2} \E_s\Bigg\|\sum_{i \in I_t} \Big(\nabla f(\vec{v}_t^i)-\nabla f(\vec{x}_t)+\nabla f(\vec{x}_t)\Big)\Bigg\|^2
\end{align*}
Observe that $\E_s\Bigg \|\sum_{i \in I_t} \Big(\tilde G(v_t^i)- \nabla f(\vec{v}_t^i)\Big)\Bigg\|^2$ is the same 
as a variance of random variable $\sum_{i \in I_t} \tilde G(v_t^i)$. Recall that we are conditioning on $\vec{x}_t$ and $\{\vec{v}_t^i|i \in I_t\}$, hence $\tilde G(v_t^i)- \nabla f(\vec{v}_t^i)$ are independent random variables
(because agents in $I_t$ compute stochastic gradients independently).
This means that because of the variance bound (\ref{eqn:variance_is_bounded_assumption_f}):
\begin{equation*}
\E_s\Bigg \|\sum_{i \in I_t} \Big(\tilde G(v_t^i)- \nabla f(\vec{v}_t^i)\Big)\Bigg\|^2 \le 
|I_t|\sigma^2.
\end{equation*}
By combining the two inequalities above we get:
\begin{align*}
\E_s[f(\vec{x}_{t+1})] &\overset{Cauchy-Schwarz}{\le} f(\vec{x}_t)-\frac{\alpha|I_t|}{p} \|\nabla f(\vec{x}_t)\|^2+\frac{\alpha}{p} \sum_{i \in I_t} \langle \nabla f(\vec{x}_t) -  \nabla f(\vec{v}_t^i), \nabla f(\vec{x}_t) \rangle \\&\quad\quad\quad\quad+\frac{2L\alpha^2|I_t|}{p^2}\sum_{i \in I_t} \Big(
\|\nabla f(\vec{v}_t^i)-\nabla f(\vec{x}_t)\|^2+\|\nabla f(\vec{x}_t)\|^2\Big)+\frac{L\alpha^2|I_t|\sigma^2}{p^2} \\ &\overset{Young}{\le}
f(\vec{x}_t)-\frac{\alpha|I_t|}{p} \|\nabla f(\vec{x}_t)\|^2+\sum_{i \in I_t} \frac{\alpha}{2p} \|\nabla f(\vec{x}_t) -  \nabla f(\vec{v}_t^i)\|^2 +\frac{\alpha|I_t|}{2p} \|\nabla f(\vec{x}_t)\|^2 
\\&\quad\quad\quad\quad+\frac{2L\alpha^2|I_t|}{p^2}\sum_{i \in I_t} \Big(
\|\nabla f(\vec{v}_t^i)-\nabla f(\vec{x}_t)\|^2+\|\nabla f(\vec{x}_t)\|^2\Big)+\frac{L\alpha^2|I_t|\sigma^2}{p^2}
\\&\overset{(\ref{eqn:smooth_gradients_f})}{\le} f(\vec{x}_t)-\frac{\alpha|I_t|}{2p} \|\nabla f(\vec{x}_t)\|^2+\frac{\alpha L^2}{2p} \sum_{i \in I_t} \|\vec{x}_t -  \vec{v}_t^i\|^2 
\\&\quad\quad\quad\quad+\frac{2L\alpha^2|I_t|}{p^2}\sum_{i \in I_t} \Big(
L\|\vec{v}_t^i-\vec{x}_t)\|^2+\|\nabla f(\vec{x}_t)\|^2\Big)+\frac{L\alpha^2|I_t|\sigma^2}{p^2}.
\end{align*}
Next, we use elastic consistency bound in the above inequality:
\begin{align*}
    \E[f(&\vec{x}_{t+1})]=\E[\E[f(\vec{x}_{t+1})|\vec{x}_t,\{\vec{v}_t^i|i \in I_t\}]] \\&\le 
    \E[f(\vec{x}_t)]-\frac{\alpha|I_t|}{2p} \E\|\nabla f(\vec{x}_t)\|^2
    +\frac{\alpha L^2}{2p}\sum_{i \in I_t} \E\|\vec{x}_t -  \vec{v}_t^i\|^2 
    \\&\quad\quad\quad\quad+\frac{2L\alpha^2|I_t|}{p^2}\sum_{i \in I_t} \Big(
\|\nabla f(\vec{v}_t^i)-\nabla f(\vec{x}_t)\|^2+\|\nabla f(\vec{x}_t)\|^2\Big)+\frac{L\alpha^2|I_t|\sigma^2}{p^2}\\ &\le
\E[f(\vec{x}_t)]-\frac{\alpha|I_t|}{2p} \E\|\nabla f(\vec{x}_t)\|^2+\frac{\alpha^3B^2 L^2|I_t|}{2p} +\frac{L\alpha^2\sigma^2|I_t|}{p^2}+\frac{2L^3\alpha^4B^2|I_t|^2}{p^2}
\\&\quad\quad\quad\quad\quad+
\frac{2L\alpha^2\E\|\nabla f(\vec{x}_t)\|^2|I_t|^2}{p^2}
.
\end{align*}
To finish the proof recall that $p/2 \le |I_t| \le p$ and $\alpha \le \frac{1}{8L}$, which if used in the inequality above gives us:
\begin{align*}
\E[f(&\vec{x}_{t+1})] \le \E[f(\vec{x}_t)]-\frac{\alpha|I_t|}{2p} \E\|\nabla f(\vec{x}_t)\|^2+\frac{\alpha^3B^2 L^2}{2} +\frac{L\alpha^2\sigma^2}{p}+2L^3\alpha^4B^2
\\&\quad\quad\quad\quad\quad+
\frac{\alpha\E\|\nabla f(\vec{x}_t)\|^2|I_t|}{4p} \\ &\le 
\E[f(\vec{x}_t)]-\frac{\alpha|I_t|}{4p} \E\|\nabla f(\vec{x}_t)\|^2+\frac{\alpha^3B^2 L^2}{2} +\frac{L\alpha^2\sigma^2}{p}+2L^3\alpha^4B^2 \\ &\le \E[f(\vec{x}_t)]-\frac{\alpha}{8} \E\|\nabla f(\vec{x}_t)\|^2+\frac{\alpha^3B^2 L^2}{2} +\frac{L\alpha^2\sigma^2}{p}+2L^3\alpha^4B^2.
\end{align*}
\end{proof}

\begin{reptheorem} {thm:nonconvexmultiple}
        Consider SGD iterations defined in 
	    (\ref{eqn:SGDmultiplesteps}) and satisfying elastic consistency bound (\ref{eqn:x-def}). For a smooth non-convex objective function $f$, whose minimum $x^*$ we are trying to find and the constant learning rate $\alpha = \frac{\sqrt{p}}{\sqrt{T}}$, where $T \ge 64L^2p$ is the number of iterations:
		
		\[
			\min_{t\in[T-1]}\mathbb{E}\|\nabla f(\vec{x}_t)\|^2 \le \frac{8(f(\vec{x}_0)-f(x^*))}{\sqrt{Tp}} +\frac{4B^2 L^2p}{T} +\frac{8L\sigma^2}{\sqrt{Tp}}+\frac{16L^3B^2p\sqrt{p}}{T\sqrt{T}}.			
		\]
\end{reptheorem}
\begin{proof}
   Observe that since $\alpha=\sqrt{p}/\sqrt{T}\le 1/8L$, by Lemma \ref{lem:multiplestepsnonconvex} 
   we have that :

   \begin{equation*}
    \E[f(\vec{x}_{t+1})] \le \E[f(\vec{x}_t)]-\frac{\alpha}{8} \E\|\nabla f(\vec{x}_t)\|^2+\frac{\alpha^3B^2 L^2}{2} +\frac{L\alpha^2\sigma^2}{p}+2L^3\alpha^4B^2.
    \end{equation*}
    By summing the above inequality for $t=0,1,..,T-1$ we get that:
    \begin{equation*}
    \sum_{t=0}^{T-1} \E[f(\vec{x}_{t+1})] \le \sum_{t=0}^{T-1} \Big ( \E[f(\vec{x}_t)]-\frac{\alpha}{8} \E\|\nabla f(\vec{x}_t)\|^2+\frac{\alpha^3B^2 L^2}{2} +\frac{L\alpha^2\sigma^2}{p}+2L^3\alpha^4B^2\Big).
    \end{equation*}
    Which can be rewritten as:
    \begin{equation*}
    \sum_{t=0}^{T-1} \frac{\alpha}{8} \E\|\nabla f(\vec{x}_t)\|^2 \le f(\vec{x}_0)-\E[f(\vec{x}_{T})] +\frac{\alpha^3B^2 L^2T}{2} +\frac{L\alpha^2\sigma^2T}{p}+2L^3\alpha^4B^2T.
    \end{equation*}
	Next, we divide by $\frac{\alpha T}{8}$ and use the fact that $\E[f(\vec{x}_{T})]\ge f(x^*)$:
	\begin{equation*}
    \sum_{t=0}^{T-1} \frac{1}{T} \E\|\nabla f(\vec{x}_t)\|^2 \le \frac{8(f(\vec{x}_0)-f(x^*))}{\alpha T} +4\alpha^2B^2 L^2+\frac{8L\alpha\sigma^2}{p}+16L^3\alpha^3B^2.
    \end{equation*}

By plugging in the value of $\alpha$ we get:
\begin{align*}
			\min_{t\in[T-1]}\mathbb{E}\|\nabla f(\vec{x}_t)\|^2 &\le 
			\sum_{t=0}^{T-1} \frac{1}{T} \E\|\nabla f(\vec{x}_t)\|^2 \le
			\frac{8(f(\vec{x}_0)-f(x^*))}{\alpha T} +4\alpha^2B^2 L^2 +\frac{8L\alpha\sigma^2}{p}+16L^3\alpha^3B^2 \\ &\le \frac{8(f(\vec{x}_0)-f(x^*))}{\sqrt{Tp}} +\frac{4B^2 L^2p}{T} +\frac{8L\sigma^2}{\sqrt{Tp}}+\frac{16L^3B^2p\sqrt{p}}{T\sqrt{T}}.
\end{align*}
\end{proof}

\subsection{Complete Proof of Convergence in the Strongly Convex Case}\label{proof:full-con}
First, we show the proof of the Lemma which is known to be useful for the analysis of convergence of $SGD$ in the strongly convex case.
\begin{lemma} \label{lem:weirdlemma}
Let $f$ be a $L$-smooth, $c$-strongly convex function, whose minimum $x^*$ we are trying to find.
For any vector $\vec{x}$, we have:

\begin{equation}
\langle \nabla f(\vec{x}), \vec{x}-x^* \rangle \ge \frac{1}{2L} \|\nabla f(\vec{x})\|^2+\frac{c}{2} \|\vec{x}-x^*\|^2.
\end{equation}
\end{lemma}

\begin{proof}
\begin{align*}
0 &\le f(\vec{x}-\frac{1}{L} \nabla f(\vec{x}))-f(x^*)=f(\vec{x}-\frac{1}{L} \nabla f(\vec{x}))-f(\vec{x})+f(\vec{x})-f(x^*) \\&
\overset{(\ref{eqn:smooth_gradients_f})}{\le} \langle \nabla f(\vec{x}), -\frac{1}{L} \nabla f(\vec{x}) \rangle+
\frac{1}{2L} \|\nabla f(\vec{x}) \|^2+f(\vec{x})-f(x^*) \\&
\overset{(\ref{eqn:strongconvex})}{\le}-\frac{1}{2L} \|\nabla f(\vec{x}) \|^2 + \langle \nabla f(\vec{x}), \vec{x}-x^* \rangle-\frac{c}{2} \|\vec{x}-x^* \|^2.
\end{align*}
The proof of the lemma follows after rearranging the terms in the above inequality.
\end{proof}

\begin{lemma} \label{lem:singlestepconvex}
Consider SGD iterations defined in 
	(\ref{eqn:SGDsinglestep}) and satisfying elastic consistency bound (\ref{eqn:x-def}). For a smooth, strongly convex objective function $f$ and the constant learning rate $\alpha \le \frac{1}{3L}$.
	We have that:
	\begin{equation}
    \E\|\vec{x}_{t+1}-x^*\|^2 \le
(1-\frac{\alpha c}{2})\E\|\vec{x}_t-x^*\|^2
+\frac{2\alpha^3B^2 L^2}{c} +3\alpha^2\sigma^2+3\alpha^4L^2B^2.
\end{equation}
	
\end{lemma}
\begin{proof}
We condition on $\vec{x}_t$ and $\vec{v}_t^i$ and calculate expectation with respect to the randomness of stochastic gradient (We call this $\E_s$).
We get that :
\begin{align*}
\E_s\|\vec{x}_{t+1}-x^*\|^2 &= \E_s\|\vec{x}_t-\alpha \tilde G(\vec{v}_t^i)-x^*\|^2=\|\vec{x}_t-x^*\|^2-2\alpha \E_s \langle \tilde G(\vec{v}_t^i), \vec{x_t}-x^*  \rangle+\alpha^2 \E_s \|\tilde G(\vec{v}_t^i)\|^2 \\& \overset{(\ref{eqn:unbiased_estimator})}{=} \|\vec{x}_t-x^*\|^2-2\alpha \langle \nabla f(\vec{x}_t), \vec{x}_t-x^* \rangle+2\alpha \langle \nabla f(\vec{x}_t) -  \nabla f(\vec{v}_t^i), \vec{x}_t-x^* \rangle \\&\quad\quad\quad\quad+\alpha^2 \E_s\|\tilde G(v_t^i)- \nabla f(\vec{v}_t^i)+\nabla f(\vec{v}_t^i)-\nabla f(\vec{x}_t)+\nabla f(\vec{x}_t)\|^2 \\ &\overset{Cauchy-Schwarz}{\le}
\|\vec{x}_t-x^*\|^2-2\alpha \langle \nabla f(\vec{x}_t), \vec{x}_t-x^* \rangle+2\alpha \langle \nabla f(\vec{x}_t) -  \nabla f(\vec{v}_t^i), \vec{x}_t-x^* \rangle \\&\quad\quad\quad\quad+3\alpha^2\Big(
\E_s\|\tilde G(\vec{v}_t^i)- \nabla f(\vec{v}_t^i)\|^2+\|\nabla f(\vec{v}_t^i)-\nabla f(\vec{x}_t)\|^2+\|\nabla f(\vec{x}_t)\|^2\Big) 
\\&\overset{\text{Lemma \ref{lem:weirdlemma} }}{\le}
\|\vec{x}_t-x^*\|^2-\frac{\alpha}{L} \| \nabla f(\vec{x}_t) \|^2-\alpha c \|\vec{x}_t-x^*\|^2+2\alpha \langle \nabla f(\vec{x}_t) -  \nabla f(\vec{v}_t^i), \vec{x}_t-x^* \rangle \\&\quad\quad\quad\quad+3\alpha^2\Big(
\E_s\|\tilde G(\vec{v}_t^i)- \nabla f(\vec{v}_t^i)\|^2+\|\nabla f(\vec{v}_t^i)-\nabla f(\vec{x}_t)\|^2+\|\nabla f(\vec{x}_t)\|^2\Big) 
\\ &\overset{Young}{\le}
\|\vec{x}_t-x^*\|^2-\frac{\alpha}{L} \| \nabla f(\vec{x}_t) \|^2-\alpha c \|\vec{x}_t-x^*\|^2
\\&\quad\quad\quad\quad
+\frac{2\alpha}{c} \|\nabla f(\vec{x}_t) -  \nabla f(\vec{v}_t^i)\|^2+\frac{c\alpha}{2} \|\vec{x}_t-x^*\|^2  \\&\quad\quad\quad\quad+3\alpha^2\Big(
\E_s\|\tilde G(\vec{v}_t^i)- \nabla f(\vec{v}_t^i)\|^2+\|\nabla f(\vec{v}_t^i)-\nabla f(\vec{x}_t)\|^2+\|\nabla f(\vec{x}_t)\|^2\Big) 
\\&\overset{(\ref{eqn:variance_is_bounded_assumption_f})}{\le}
(1-\frac{\alpha c}{2})\|\vec{x}_t-x^*\|^2-\frac{\alpha}{L} \| \nabla f(\vec{x}_t) \|^2
+\frac{2\alpha}{c} \|\nabla f(\vec{x}_t) -  \nabla f(\vec{v}_t^i)\|^2 \\&\quad\quad\quad\quad+3\alpha^2\Big(
\sigma^2+\|\nabla f(\vec{v}_t^i)-\nabla f(\vec{x}_t)\|^2+\|\nabla f(\vec{x}_t)\|^2\Big) 
\\&\overset{(\ref{eqn:smooth_gradients_f})}{\le}
(1-\frac{\alpha c}{2})\|\vec{x}_t-x^*\|^2-\frac{\alpha}{L} \| \nabla f(\vec{x}_t) \|^2
+\frac{2\alpha L^2}{c} \|\vec{x}_t -  \vec{v}_t^i\|^2 \\&\quad\quad\quad\quad+3\alpha^2\Big(
\sigma^2+L^2\|\vec{v}_t^i-\vec{x}_t\|^2+\|\nabla f(\vec{x}_t)\|^2\Big) \\ &\le
(1-\frac{\alpha c}{2})\|\vec{x}_t-x^*\|^2
+\frac{2\alpha L^2}{c} \|\vec{x}_t -  \vec{v}_t^i\|^2 +3\alpha^2\Big(
\sigma^2+L^2\|\vec{v}_t^i-\vec{x}_t\|^2\Big).
\end{align*}
Where, in the last inequality, we used $\alpha \le \frac{1}{3L}$.
Next, we use elastic consistency bound :
\begin{align*}
    \E\|\vec{x}_{t+1}-x^*\|^2&=\E[\E[\|\vec{x}_{t+1}-x^*\|^2|\vec{x}_t,\vec{v}_t^i]] \\&\le 
    (1-\frac{\alpha c}{2})\E\|\vec{x}_t-x^*\|^2
+\frac{2\alpha L^2}{c} \E\|\vec{x}_t -  \vec{v}_t^i\|^2 +3\alpha^2\Big(
\sigma^2+L^2\E\|\vec{v}_t^i-\vec{x}_t\|^2\Big) \\ &\le
(1-\frac{\alpha c}{2})\E\|\vec{x}_t-x^*\|^2
+\frac{2\alpha^3B^2 L^2}{c} +3\alpha^2\sigma^2+3\alpha^4L^2B^2.
\end{align*}
\end{proof}

\begin{reptheorem} {thm:convexsingle}
        Consider SGD iterations defined in 
	    (\ref{eqn:SGDsinglestep})  and satisfying elastic consistency bound (\ref{eqn:x-def}). For a smooth, strongly convex function $f$, whose minimum $x^*$ we are trying to find and the constant learning rate $\alpha = \frac{2\log{T}}{cT}$, where $T \ge \frac{144L^2}{c^2}$ is a number of iterations:
	    \begin{equation*}
	      \E\|\vec{x}_{T}-x^*\|^2 \le   \frac{\|\vec{x}_0-x^*\|^2}{T}+\frac{16\log^2{T}L^2B^2}{c^4T^2}+\frac{12\sigma^2\log{T}}{T}+
        \frac{48\log^3{T}B^2L^2}{c^4T^3}.
	    \end{equation*}
\end{reptheorem}
\begin{proof}
   Observe that since $\alpha \le \frac{2\log{T}}{Tc} \le \frac{4}{\sqrt{T}c} \le 1/3L$,
   By Lemma \ref{lem:singlestepconvex},
   we have that:

   \begin{equation*}
	   \E\|\vec{x}_{t+1}-x^*\|^2\le
(1-\frac{\alpha c}{2})\E\|\vec{x}_t-x^*\|^2
+\frac{2\alpha^3 L^2B^2}{c}
+3\alpha^2\sigma^2+3L^2\alpha^4B^2.
    \end{equation*}
   By using induction, it is easy to show that:
   \begin{align*}
       \E\|\vec{x}_{T}-x^*\|^2 &\le (1-\frac{\alpha c}{2})^{T} \|\vec{x}_0-x^*\|^2+
       \sum_{t=0}^{T-1}(1-\frac{\alpha c}{2})^t\Big(\frac{2\alpha^3 L^2B^2}{c}+3\alpha^2\sigma^2+3L^2\alpha^4B^2 \Big)
        \\ &\le
        (1-\frac{\alpha c}{2})^{T} \|\vec{x}_0-x^*\|^2+
       \sum_{t=0}^{\infty}(1-\frac{\alpha c}{2})^t\Big(\frac{2\alpha^3 L^2B^2}{c}+3\alpha^2\sigma^2+3L^2\alpha^4B^2 \Big) \\ &\le
        e^{-\frac{\alpha cT}{2}} \|\vec{x}_0-x^*\|^2+
       \frac{4\alpha^2 L^2B^2}{c^2}+\frac{6\alpha\sigma^2}{c}+\frac{6L^2\alpha^3B^2}{c} \\ &=
        \frac{\|\vec{x}_0-x^*\|^2}{T}+\frac{16\log^2{T}L^2B^2}{c^4T^2}+\frac{12\sigma^2\log{T}}{T}+
        \frac{48\log^3{T}B^2L^2}{c^4T^3}.
    \end{align*}
   
\end{proof}

\begin{lemma} \label{lem:multiplestepsconvex}
Consider SGD iterations defined in 
	(\ref{eqn:SGDmultiplesteps}) and satisfying elastic consistency bound (\ref{eqn:x-def}). For a smooth, strongly  convex objective function $f$ and the constant learning rate $\alpha \le \frac{1}{4L}$.
	We have that:
	\begin{equation}
 \E\|\vec{x}_{t+1}-x^*\|^2\le
(1-\frac{\alpha c}{4})\E\|\vec{x}_t-x^*\|^2
+\frac{2\alpha^3 L^2B^2}{c}
+\frac{2\alpha^2\sigma^2}{p}+4L^2\alpha^4B^2.
\end{equation}
\end{lemma}
\begin{proof}
We condition on $\vec{x}_t$ and $\{\vec{v}_t^i|i \in I_t\}$ and calculate expectation with respect to the randomness of stochastic gradient(We call this $\E_s$).
By descent lemma we get that :
\begin{align*}
\E_s\|\vec{x}_{t+1}-x^*\|^2 &= \E_s\|\vec{x}_t-\frac{\alpha}{p} \sum_{i \in I_t}\tilde G(\vec{v}_t^i)-x^*\|^2 \\&=\|\vec{x}_t-x^*\|^2-\frac{2\alpha}{p} \sum_{i \in I_t} \E_s \langle \tilde G(\vec{v}_t^i), \vec{x_t}-x^*  \rangle+\frac{\alpha^2}{p^2} \E_s \|\sum_{i \in I_t}\tilde G(\vec{v}_t^i)\|^2 \\& \overset{(\ref{eqn:unbiased_estimator})}{=} \|\vec{x}_t-x^*\|^2-\frac{2\alpha|I_t|}{p} \langle \nabla f(\vec{x}_t), \vec{x}_t-x^* \rangle+\frac{2\alpha}{p} \sum_{i \in I_t} \langle \nabla f(\vec{x}_t) -  \nabla f(\vec{v}_t^i), \vec{x}_t-x^* \rangle \\&\quad\quad\quad\quad+\frac{\alpha^2}{p^2}\E_s\|\sum_{i \in I_t}\Big(\tilde G(v_t^i)- \nabla f(\vec{v}_t^i)+\nabla f(\vec{v}_t^i)-\nabla f(\vec{x}_t)+\nabla f(\vec{x}_t)\Big)\|^2 \\ &\overset{Cauchy-Schwarz}{\le}
\|\vec{x}_t-x^*\|^2-\frac{2\alpha|I_t|}{p} \langle \nabla f(\vec{x}_t), \vec{x}_t-x^* \rangle
\\&\quad\quad\quad\quad+\frac{2\alpha^2}{p^2}\E_s\|\sum_{i \in I_t}\Big(\tilde G(v_t^i)- \nabla f(\vec{v}_t^i)\Big)\|^2
\\&\quad\quad\quad\quad+\frac{2\alpha^2}{p^2}\E_s\|\sum_{i \in I_t}\Big(\nabla f(\vec{v}_t^i)-\nabla f(\vec{x}_t)+\nabla f(\vec{x}_t)\Big)\|^2
.
\end{align*}
Observe that $\E_s\Bigg \|\sum_{i \in I_t} \Big(\tilde G(v_t^i)- \nabla f(\vec{v}_t^i)\Big)\Bigg\|^2$ is the same 
as a variance of random variable $\sum_{i \in I_t} \tilde G(v_t^i)$. Recall that we are conditioning on $\vec{x}_t$ and $\{\vec{v}_t^i|i \in I_t\}$, hence $\tilde G(v_t^i)- \nabla f(\vec{v}_t^i)$ are independent random variables
(because agents in $I_t$ compute stochastic gradients independently).
This means that because of the variance bound (\ref{eqn:variance_is_bounded_assumption_f}):
\begin{equation*}
\E_s\Bigg \|\sum_{i \in I_t} \Big(\tilde G(v_t^i)- \nabla f(\vec{v}_t^i)\Big)\Bigg\|^2 \le 
|I_t|\sigma^2.
\end{equation*}
Also by Cauchy-Schwarz inequality we get that 
\begin{align*}
\E_s\|\sum_{i \in I_t}\Big(\nabla f(\vec{v}_t^i)-\nabla f(\vec{x}_t)+\nabla f(\vec{x}_t)\Big)\|^2&=
\|\sum_{i \in I_t}\Big(\nabla f(\vec{v}_t^i)-\nabla f(\vec{x}_t)+\nabla f(\vec{x}_t)\Big)\|^2 \\&\le
2|I_t| \sum_{i \in I_t}\Big(\|\nabla f(\vec{v}_t^i)-\nabla f(\vec{x}_t)\|^2+\|\nabla f(\vec{x}_t)\|^2\Big).
\end{align*}
By combining the above three inequalities above we get:

\begin{align*}
\E_s\|\vec{x}_{t+1}-x^*\|^2&\overset{\text{Lemma \ref{lem:weirdlemma} }}{\le}
\|\vec{x}_t-x^*\|^2-\frac{\alpha|I_t|}{Lp} \| \nabla f(\vec{x}_t) \|^2-\frac{\alpha c |I_t|}{p} 
\\&\quad\quad\quad\quad+\frac{2\alpha}{p} \sum_{i \in I_t} \langle \nabla f(\vec{x}_t) -  \nabla f(\vec{v}_t^i), \vec{x}_t-x^* \rangle 
 \\&\quad\quad\quad\quad+\frac{4\alpha^2|I_t|}{p^2}\sum_{i \in I_t}\Big(
\|\nabla f(\vec{v}_t^i)-\nabla f(\vec{x}_t)\|^2+\|\nabla f(\vec{x}_t)\|^2\Big)+\frac{2\alpha^2\sigma^2|I_t|}{p^2}
\\ &\overset{Young}{\le}
\|\vec{x}_t-x^*\|^2-\frac{\alpha|I_t|}{Lp} \| \nabla f(\vec{x}_t) \|^2-\frac{\alpha c |I_t|}{p} 
\\&\quad\quad\quad\quad
+\frac{2\alpha}{cp} \sum_{i \in |I_t}\|\nabla f(\vec{x}_t) -  \nabla f(\vec{v}_t^i)\|^2+\frac{c\alpha |I_t|}{2p} \|\vec{x}_t-x^*\|^2  \\&\quad\quad\quad\quad+\frac{4\alpha^2|I_t|}{p^2}\sum_{i \in I_t}\Big(
\|\nabla f(\vec{v}_t^i)-\nabla f(\vec{x}_t)\|^2+\|\nabla f(\vec{x}_t)\|^2\Big)+\frac{2\alpha^2\sigma^2|I_t|}{p^2}
\\&\overset{(\ref{eqn:smooth_gradients_f})}{\le}
(1-\frac{\alpha c |I_t|}{2p})\|\vec{x}_t-x^*\|^2-\frac{\alpha |I_t|}{Lp} \| \nabla f(\vec{x}_t) \|^2
+\frac{2\alpha L^2}{cp} \sum_{i \in I_t}\|\vec{x}_t -  \vec{v}_t^i\|^2 \\&\quad\quad\quad\quad+
\frac{4\alpha^2|I_t|}{p^2}\sum_{i \in I_t} \Big(
L^2\|\vec{v}_t^i-\vec{x}_t\|^2+\|\nabla f(\vec{x}_t)\|^2\Big) +\frac{2\alpha^2\sigma^2|I_t|}{p^2} \\ &\le
(1-\frac{\alpha c|I_t|}{2p})\|\vec{x}_t-x^*\|^2
+\frac{2\alpha L^2}{cp} \sum_{i \in I_t} \|\vec{x}_t -  \vec{v}_t^i\|^2 +\frac{4\alpha^2|I_t|}{p^2} \sum_{i \in I_t} \Big(
L^2\|\vec{v}_t^i-\vec{x}_t\|^2\Big)\\&\quad\quad\quad\quad+\frac{2\alpha^2\sigma^2}{p}.
\end{align*}
Where, in the last inequality, we used $\alpha \le \frac{1}{4L}$ and $|I_t| \le p$.
Next, we use elastic consistency bound and $p/2 \le |I_t| \le p$:
\begin{align*}
    \E\|\vec{x}_{t+1}&-x^*\|^2=\E[\E[\|\vec{x}_{t+1}-x^*\|^2|\vec{x}_t,\{\vec{v}_t^i|i \in I_t\}]] \\&\le 
    (1-\frac{\alpha c|I_t|}{2p})\E\|\vec{x}_t-x^*\|^2
+\frac{2\alpha L^2}{cp} \sum_{i \in I_t} \E\|\vec{x}_t -  \vec{v}_t^i\|^2
+\frac{4\alpha^2|I_t|L^2}{p^2} \sum_{i \in I_t} \E\|\vec{v}_t^i-\vec{x}_t\|^2+\frac{2\alpha^2\sigma^2}{p} \\ &\le
(1-\frac{\alpha c}{4})\E\|\vec{x}_t-x^*\|^2
+\frac{2\alpha^3 L^2B^2}{c}
+\frac{2\alpha^2\sigma^2}{p}+4L^2\alpha^4B^2.
\end{align*}
\end{proof}

\begin{reptheorem} {thm:convexmultiple}
        Consider SGD iterations defined in (\ref{eqn:SGDmultiplesteps})  and satisfying elastic consistency bound (\ref{eqn:x-def}). For a smooth, strongly convex function $f$, whose minimum $x^*$ we are trying to find and the constant learning rate $\alpha = \frac{2(\log{T}+\log{p})}{cT}$, where $T \ge \frac{256L^2p}{c^2}$ is a number of iterations:
        \begin{align*}
        \E\|\vec{x}_{T}-x^*\|^2 \le \frac{\|\vec{x}_0-x^*\|^2}{Tp}&+\frac{16(\log{T}+\log{p})^2L^2B^2}{c^4T^2}+\frac{12\sigma^2(\log{T}+\log{p})}{Tp}\\&+
        \frac{48(\log{T}+\log{p})^3B^2L^2}{c^4T^3}.
        \end{align*}
\end{reptheorem}
\begin{proof}
   Observe that since $\alpha \le \frac{2\log{Tp}{p}}{Tc} \le \frac{4\sqrt{p}}{\sqrt{T}c} \le 1/3L$,
   By Lemma \ref{lem:multiplestepsconvex},
   we have that:

   \begin{equation*}
	   \E\|\vec{x}_{t+1}-x^*\|^2\le
(1-\frac{\alpha c}{4})\E\|\vec{x}_t-x^*\|^2
+\frac{2\alpha^3 L^2B^2}{c}
+\frac{2\alpha^2\sigma^2}{p}+4L^2\alpha^4B^2.
    \end{equation*}
   By using induction, it is easy to show that:
   \begin{align*}
       \E&\|\vec{x}_{T}-x^*\|^2 \le (1-\frac{\alpha c}{4})^{T} \|\vec{x}_0-x^*\|^2+
       \sum_{t=0}^{T-1}(1-\frac{\alpha c}{4})^t\Big(\frac{2\alpha^3 L^2B^2}{c}+\frac{2\alpha^2\sigma^2}{p}+4L^2\alpha^4B^2 \Big)
        \\ &\le
        (1-\frac{\alpha c}{2})^{T} \|\vec{x}_0-x^*\|^2+
       \sum_{t=0}^{\infty}(1-\frac{\alpha c}{4})^t \Big(\frac{2\alpha^3 L^2B^2}{c}+\frac{2\alpha^2\sigma^2}{p}+4L^2\alpha^4B^2 \Big)\\ &\le
        e^{-\frac{\alpha cT}{2}} \|\vec{x}_0-x^*\|^2+
       \frac{8\alpha^2 L^2B^2}{c^2}+\frac{8\alpha\sigma^2}{pc}+\frac{16L^2\alpha^3B^2}{c} \\ &=
        \frac{\|\vec{x}_0-x^*\|^2}{Tp}+\frac{16(\log{T}+\log{p})^2L^2B^2}{c^4T^2}+\frac{12\sigma^2(\log{T}+\log{p})}{Tp}+
        \frac{48(\log{T}+\log{p})^3B^2L^2}{c^4T^3}.
    \end{align*}
   
\end{proof}

\section{Elastic Consistency Bounds}

\subsection{Synchronous message passing with crash faults and variance bound} 

Consider the algorithm \ref{alg:crashfaultsvariance} for node $i$ at iteration $t$.
Recall that $\mathcal{L}_t^i$ is a set of nodes which send their computed gradients
to node $i$ at iteration $t$ (for the convenience assume that $\mathcal{L}_{-1}^i=\mathcal{P}$). We assume that $i \in \mathcal{L}_t^i$.
In the model with crash faults we have that if some node $j \notin \mathcal{L}_t^i$, then this means that node $j$ has crashed.
Further, if $j \in \mathcal{L}_{t-1}^i \setminus \mathcal{L}_t^i$, this means that $j$ crashed during iteration $t$.
More specifically this means that $j$ generated it's stochastic gradient and crashed before sending it to $i$.
In this case $i$ uses it's own stochastic gradient $\tg(\vec{v}_t^i)$ as a substitute.
\begin{algorithm}[h]
	\caption{Iteration $t$ at node $i \in \mathcal{P}$, for Crash faults model with variance} \label{alg:crashfaultsvariance}
	Compute $\tg\left(\vec{v}^{i}_t\right)$;~\tcp*[h]{Compute SG using the local view.}\\
	Broadcast $\tg\left(\vec{v}^{i}_t\right)$;~\tcp*[h]{Broadcast SG to all the nodes.}\\
	$\tg\gets0$;~~\tcp*[h]{Prepare to collect 
	stochastic gradients from the nodes.}\\
	\For {$j \in \mathcal{L}_t^i$}{
			$\tg \gets \tg + 
				\tg\left(\vec{v}^{j}_t\right)$
	}					
	\For {$j \in \mathcal{L}_{t-1}^i \setminus \mathcal{L}_t^i$} { 
	    $\tg \gets \tg + 
				\tg\left(\vec{v}^{i}_t\right)$ ~\tcp*[h]{if $j$ crashed during iteration $t$, substitute it's SG.}
	}
	$\vec{v}^i_{t+1} \gets \vec{v}^i_{t} - \frac{\alpha}{p}\tg$;~\tcp*[h]{
	Update \modl vector.}
\end{algorithm}
We define the auxiliary variable $\vec{x}_t$ as sum of all stochastic gradients which were generated up to and excluding iteration $t$ and sent to at least one node multiplied by $-\frac{\alpha}{p}$. 
This means that we use the update rule (\ref{eqn:SGDmultiplesteps}) with $I_t = \{j|\exists i ,j \in \mathcal{L}_t^i\}$. Notice that since at most $p/2$ nodes can crash we have that $p/2 \le |I_t| \le p$.
Let $f_t \le f$ be the total number of crashed nodes up to and including iteration $t$
We proceed by proving the elastic consistency bound with $B=\frac{3f\sigma}{p}$:
\begin{lemma} \label{lem:crashfaultsvariance}
	In a synchronous message-passing system consisting of $p$ nodes with 
	failure bound $f \le p/2$, For any iteration $t$ and node $i$ which has not crashed yet (it computes stochastic gradient at iteration $t$), if $\alpha \le \frac{1}{6L}$(Notice that this is consistent with the upper bound used in the convergence proofs) we have: 
	\begin{equation}
	    \E \|\vec{v}_t^i-\vec{x}_t\|^2 \le \frac{9\alpha^2 f_t^2 \sigma^2}{p^2}.
	\end{equation}
\end{lemma}
\begin{proof}
We prove the claim by using induction on the number of iterations.
Base case holds trivially. For the induction step assume that the lemma holds for iteration $t$:
Let $h=|\mathcal{L}_{t}^i \setminus \mathcal{L}_{t+1}^i|$ be the number of nodes which crashed 
during iteration $t+1$ (we assume the worst case, which means that every node which crashed during
iteration $t+1$ failed to send stochastic gradient to the node $i$).
Assume that $h>0$ since otherwise the proof is trivial.
Notice that 
\begin{equation}
    \vec{x}_{t+1}-\vec{v}_{t+1}^i=\vec{x}_{t}-\vec{v}_{t}^i+\sum_{j \in \mathcal{L}_{t}^i \setminus \mathcal{L}_{t+1}^i} 
    \frac{\alpha}{p}(\tg (\vec{v}_t^i)-\tg (\vec{v}_t^j))
\end{equation}
Hence we get that
\begin{align*}
\E \|\vec{x}_{t+1}&-\vec{v}_{t+1}^i \|^2=\E \Big\|\vec{x}_{t}-\vec{v}_{t}^i+\sum_{j \in \mathcal{L}_{t}^i \setminus \mathcal{L}_{t+1}^i} \frac{\alpha}{p}(\tg (\vec{v}_t^i)-\tg (\vec{v}_t^j)) \Big\|^2
\\ &\overset{Young}{\le} (1+\frac{h}{f_t}) \E \|\vec{x}_{t}-\vec{v}_{t}^i\|^2+(1+\frac{f_t}{h}) 
\E \Big\|\sum_{j \in \mathcal{L}_{t}^i \setminus \mathcal{L}_{t+1}^i} \frac{\alpha}{p}(\tg (\vec{v}_t^i)-\tg (\vec{v}_t^j)) \Big\|^2 
\\ &\overset{Cauchy-Schwarz}{\le} 
(1+\frac{h}{f_t}) \E \|\vec{x}_{t}-\vec{v}_{t}^i\|^2+(1+\frac{f_t}{h})h\frac{\alpha^2}{p^2} 
\sum_{j \in \mathcal{L}_{t}^i \setminus \mathcal{L}_{t+1}^i} \E\|\tg (\vec{v}_t^i)-\tg (\vec{v}_t^j)\|^2\\&=
(1+\frac{h}{f_t}) \E \|\vec{x}_{t}-\vec{v}_{t}^i\|^2\\&\quad\quad\quad+(1+\frac{f_t}{h})h\frac{\alpha^2}{p^2} 
\sum_{j \in \mathcal{L}_{t}^i \setminus \mathcal{L}_{t+1}^i} 
\E\|\tg (\vec{v}_t^i)-\nabla f (\vec{v}_t^i)+\nabla f (\vec{v}_t^i)-\nabla f (\vec{v}_t^j)+\nabla f (\vec{v}_t^j)-\tg (\vec{v}_t^j)\|^2 \\&\overset{Cauchy-Schwarz}{\le} 
(1+\frac{h}{f_t}) \E \|\vec{x}_{t}-\vec{v}_{t}^i\|^2\\&\quad\quad\quad\quad\quad\quad+(1+\frac{f_t}{h})h\frac{\alpha^2}{p^2} 
\sum_{j \in \mathcal{L}_{t}^i \setminus \mathcal{L}_{t+1}^i}\Big(
3\E\|\tg (\vec{v}_t^i)-\nabla f (\vec{v}_t^i)\|^2+\\&\quad\quad\quad\quad\quad\quad\quad\quad\quad\quad\quad\quad\quad\quad\quad\quad\quad\quad
3\E\|\nabla f (\vec{v}_t^i)-\nabla f (\vec{v}_t^j)\|^2+3\E\|\nabla f (\vec{v}_t^j)-\tg (\vec{v}_t^j)\|^2\Big)
\\ &\overset{(\ref{eqn:variance_is_bounded_assumption_f})}{\le}
(1+\frac{h}{f_t}) \E \|\vec{x}_{t}-\vec{v}_{t}^i\|^2+(1+\frac{f_t}{h})h\frac{\alpha^2}{p^2} 
\sum_{j \in \mathcal{L}_{t}^i \setminus \mathcal{L}_{t+1}^i}\Big(
3\E\|\nabla f(\vec{v}_t^i)-\nabla f (\vec{v}_t^j)\|^2+6\sigma^2\Big)
\\&\overset{(\ref{eqn:smooth_gradients_f})}{\le}
(1+\frac{h}{f_t}) \E \|\vec{x}_{t}-\vec{v}_{t}^i\|^2+(1+\frac{f_t}{h})h\frac{\alpha^2}{p^2} 
\sum_{j \in \mathcal{L}_{t}^i \setminus \mathcal{L}_{t+1}^i}\Big(
3L^2\E\|\vec{v}_t^i-\vec{x}_t+\vec{x}_t-\vec{v}_t^j\|^2+6\sigma^2\Big)
\\&\overset{Cauchy-Schwarz}{\le} 
(1+\frac{h}{f_t}) \E \|\vec{x}_{t}-\vec{v}_{t}^i\|^2\\&\quad\quad\quad\quad\quad+(1+\frac{f_t}{h})h\frac{\alpha^2}{p^2} 
\sum_{j \in \mathcal{L}_{t}^i \setminus \mathcal{L}_{t+1}^i}\Big(
6L^2(\E\|\vec{v}_t^i-\vec{x}_t\|+\E\|\vec{x}_t-\vec{v}_t^j\|^2)+6\sigma^2\Big)
\end{align*}
Next, we use the assumption that lemma holds for nodes at iteration $t$.
We get that:
\begin{align*}
 \E \|\vec{x}_{t+1}&-\vec{v}_{t+1}^i \|^2 \le (1+\frac{h}{f_t}) \frac{9\alpha^2\sigma^2f_t^2}{p^2}\\&\quad\quad\quad\quad\quad+(1+\frac{f_t}{h})h\frac{9\alpha^2}{p^2} 
\sum_{j \in \mathcal{L}_{t}^i \setminus \mathcal{L}_{t+1}^i}\Big(
12L^2\frac{\alpha^2\sigma^2f_t^2}{p^2}+6\sigma^2\Big) \\ &=
(1+\frac{h}{f_t}) \frac{9\alpha^2\sigma^2f_t^2}{p^2}+(1+\frac{f_t}{h})h^2\frac{\alpha^2}{p^2} 
\Big(12L^2\frac{9\alpha^2\sigma^2f_t^2}{p^2}+6\sigma^2\Big)
\end{align*}
Finally, we use $f_t \le f \le p$ and $\alpha \le \frac{1}{6L}$, to get:
\begin{align*}
 \E \|\vec{x}_{t+1}&-\vec{v}_{t+1}^i \|^2 \le 
(1+\frac{h}{f_t}) \frac{9\alpha^2\sigma^2f_t^2}{p^2}+(1+\frac{f_t}{h})h^2\frac{\alpha^2}{p^2} 
\Big(3\sigma^2+6\sigma^2\Big)\\&=\frac{9\alpha^2\sigma^2}{p^2}(f_t+h)^2=\frac{9\alpha^2f_{t+1}^2\sigma^2}{p^2}.
\end{align*}
\end{proof}

\subsection{Crash faults with second moment bound}

Consider the algorithm \ref{alg:crashfaults} for node $i$ at iteration $t$.
Recall that $\mathcal{L}_t^i$ is a set of nodes which send their computed gradients
to node $i$ at iteration $t$. We assume that $i \in \mathcal{L}_t^i$.
In the model with crash faults we have that if some node $j \notin \mathcal{L}_t^i$, then this means that node $j$ has crashed.

\begin{algorithm}[h]
	\caption{Iteration $t$ at node $i \in \mathcal{P}$, for Crash faults model} \label{alg:crashfaults}
	Compute $\tg\left(\vec{v}^{i}_t\right)$;~\tcp*[h]{Compute SG using the local view.}\\
	Broadcast $\tg\left(\vec{v}^{i}_t\right)$;~\tcp*[h]{Broadcast SG to all the nodes.}\\
	$\tg\gets0$;~~\tcp*[h]{Prepare to collect 
	stochastic gradients from the nodes.}\\
	\For {$j \in \mathcal{L}_t^i$}{
			$\tg \gets \tg + 
				\tg\left(\vec{v}^{j}_t\right)$
	}					
	$\vec{v}^i_{t+1} \gets \vec{v}^i_{t} - \frac{\alpha}{p}\tg$;~\tcp*[h]{
	Update \modl vector.}
\end{algorithm}
We define the auxiliary variable $\vec{x}_t$ as sum of all stochastic gradients which were generated up to and excluding iteration $t$ and sent 
to at least one node multiplied by $-\frac{\alpha}{p}$. 
This means that we use the update rule (\ref{eqn:SGDmultiplesteps}) with $I_t = \{j|\exists i ,j \in \mathcal{L}_t^i\}$. Notice that since at most $p/2$ nodes can crash we have that $p/2 \le |I_t| \le p$.
We proceed by proving the elastic consistency bound with $B=\frac{fM}{p}$:
\begin{lemma} \label{lem:crashfaults}
	In a synchronous message-passing system consisting of $p$ nodes with 
	failure bound $f \le p/2$, For any iteration $t$ and node $i$ which has not crashed yet (it computes stochastic gradient at iteration $t$) we have: 
	\begin{equation}
	    \E \|\vec{v}_t^i-\vec{x}_t\|^2 \le \frac{\alpha^2 f^2 M^2}{p^2}
	\end{equation}
\end{lemma}
\begin{proof}
Notice that $\vec{v}_t^i-\vec{x}_t=\sum_{s=0}^{t-1} \sum_{j \in I_s \setminus \mathcal{L}_s^i} \frac{\alpha}{p} \tg(\vec{v}_s^j)$, where $j \in I_s \setminus \mathcal{L}_s^i$ means that 
node $j$ crashed at iteration $s$, before sending it's stochastic gradient to node $i$.
Since each node can crash at most once, we have that $\sum_{s=0}^{t-1} |I_s \setminus \mathcal{L}_s^i| \le f$. We get that:
\begin{align*}
\E \|\vec{v}_t^i-\vec{x}_t\|^2&=\E\|\sum_{s=0}^{t-1} \sum_{j \in I_s \setminus \mathcal{L}_s^i} \frac{\alpha}{p} \tg(\vec{v}_s^j)\|^2 \overset{Cauchy-Schwarz}{\le} f\sum_{s=0}^{t-1} \sum_{j \in I_s \setminus \mathcal{L}_s^i} \frac{\alpha^2}{p^2} \E\|\tg(\vec{v}_s^j)\|^2 \\ &\overset{(\ref{eqn:grad_is_bounded_assumption_f})}{\le} f\sum_{s=0}^{t-1} \sum_{j \in I_s \setminus \mathcal{L}_s^i} \frac{\alpha^2}{p^2} M^2 \le \frac{\alpha^2 f^2 M^2}{p^2}.
\end{align*}
\end{proof}
	
\subsection{Synchronous message passing with message-omission failures}
\begin{algorithm}[h]
	\caption{Iteration $t$ at node $i \in \mathcal{P}$, for message delays} \label{alg:messagedelays}
	Compute $\tg\left(\vec{v}^{i}_t\right)$;~\tcp*[h]{Compute SG using the local view.}\\
	Broadcast $\tg\left(\vec{v}^{i}_t\right)$;~\tcp*[h]{Broadcast SG to all the nodes.}\\
	$\tg\gets0$;~~\tcp*[h]{Prepare to collect 
	stochastic gradients from the nodes.}\\
	\For {$(s,j) \in \mathcal{L}_t^i$}{
			$\tg \gets \tg + 
				\tg\left(\vec{v}^{j}_s\right)$
	}					
	$\vec{v}^i_{t+1} \gets \vec{v}^i_{t} - \frac{\alpha}{p}\tg$;~\tcp*[h]{
	Update \modl vector.}
\end{algorithm}

Consider algorithm \ref{alg:messagedelays}.
Here, we have that the set $\mathcal{L}_t^i$ might contain delayed messages (which were generated before 
the iteration $t$). Hence, we assume that  $\mathcal{L}_t^i$ is a set of pairs $(s,j)$ such that the stochastic gradient generated by the node $j$ at iteration $s \le t$ was delivered to the node $i$ at iteration $t$.
The important thing is that, at any iteration, the number of delayed messages (which might or might not be delivered in the future) is at most $f$.
We assume that nodes always "send" stochastic gradients to themselves without any delay, that is,
for any node $i$, $(t,i) \in \mathcal{L}_t^i$.
Further, let $\mathcal{K}_{t}^i=\bigcup_{t' \le t} \mathcal{L}_{t'}^i$. Notice that,
$\vec{v}_t^i=-\sum_{(s,j) \in\mathcal{K}_{t-1}^i} \frac{\alpha}{p} \tg\left(\vec{v}^{j}_s\right)$.
The auxiliary variable $\vec{x}_t=-\sum_{s=0}^{t-1} \sum_{j \in \mathcal{P}} \frac{\alpha}{p} \tg\left(\vec{v}^{j}_s\right)$.
This means that we use the update rule (\ref{eqn:SGDmultiplesteps}) with $I_t = \mathcal{P}$.
We following lemma proves the elastic consistency bound with $B=\frac{fM}{p}$:
\begin{lemma} \label{lem:messageomissionfailures}
	In a synchronous message-passing system consisting of $p$ nodes where message-omission  
	failures can be upper bounded by $f$, we have that for any iteration $t$ and node $i$:
	\begin{equation}
	    \E \|\vec{v}_t^i-\vec{x}_t\|^2 \le \frac{\alpha^2 f^2 M^2}{p^2}
	\end{equation}
\end{lemma}
\begin{proof}
For any iteration $t$, let $\mathcal{C}_t=\{0,1,...,t\} \times \mathcal{P}$, so that
$\vec{x}_t=-\sum_{(s,j) \in \mathcal{C}_{t-1}} \frac{\alpha}{p} \tg\left(\vec{v}^{j}_s\right)$.
Notice that $\vec{v}_t^i-\vec{x}_t=\sum_{(s,j) \in \mathcal{C}_{t-1} \setminus \mathcal{K}_{t-1}^i} \frac{\alpha}{p} \tg ({\vec{v}_s^j})$.
If $(s,j) \in \mathcal{C}_{t-1} \setminus \mathcal{K}_{t-1}^i$, then we have that the stochastic 
gradient generated by node $j$ at iteration $s$, is still not delivered to node $i$ before iteration $t$ starts. Since this can happen to at most $f$ messages, we have that $|\mathcal{C}_{t-1} \setminus \mathcal{K}_{t-1}^i| \le f$. Hence, we get that:
\begin{align*}
\E \|\vec{v}_t^i-\vec{x}_t\|^2&=\E\Bigg\|\sum_{(s,j) \in \mathcal{C}_{t-1} \setminus \mathcal{K}_{t-1}} \frac{\alpha}{p} \tg(\vec{v}_s^j)\Bigg\|^2 \overset{Cauchy-Schwarz}{\le}f \sum_{(s,j) \in \mathcal{C}_{t-1} \setminus \mathcal{K}_{t-1}}  \frac{\alpha^2}{p^2} \E\|\tg(\vec{v}_s^j)\|^2 \\ &\overset{(\ref{eqn:grad_is_bounded_assumption_f})}{\le} f \sum_{(s,j) \in \mathcal{C}_{t-1} \setminus \mathcal{K}_{t-1}} \frac{\alpha^2}{p^2} M^2 \le \frac{\alpha^2 f^2 M^2}{p^2}.
\end{align*}
\end{proof}
\paragraph{Removing the second moment bound.} Note that we can use the similar approach as in the case of crash faults and replace $M$ in the elastic consistency bound with $O(\sigma)$. More precisely, for node $i$ at iteration $t$, if  stochastic gradient $\tg(v_t^j)$ from node $j$ is delayed node $i$ uses it's own
stochastic gradient $\tg(v_t^i)$ and later corrects the error once it receives the actual gradient. The algorithm \ref{alg:elastic} described in section \ref{sec:elasticscheduling} can be viewed as example of this, in the case when $f=p-1$ and messages are delayed by at most one iteration.
In this case, error correction will ensure that 
\begin{equation*}
\E \|\vec{v}_t^i-\vec{x}_t\|^2=\E\Bigg\|\sum_{(s,j) \in \mathcal{C}_{t-1} \setminus \mathcal{K}_{t-1}} \frac{\alpha}{p} (\tg(\vec{v}_s^i)-\tg(\vec{v}_s^j))\Bigg\|^2.
\end{equation*}
and then we can use the similar approach as in the proofs
of Lemmas \ref{lem:elasticscheduling} and $\ref{lem:crashfaultsvariance}$ to derive 
elastic consistency bound with $B=O(\frac{f \sigma}{p})$.

\subsection{Asynchronous message passing}
This case is very similar to the case with message-ommision failures.
Consider algorithm \ref{alg:messagedelays} again.
As before, $\mathcal{L}_t^i$ is a set of pairs $(s,j)$ such that the stochastic gradient generated by the node $j$ at iteration $t-\tau_{max} \le s \le t$ was delivered to the node $i$ at iteration $t$(Recall that messages can be delayed by at most $\tau_{max}$ iterations).
We assume that nodes always "send" stochastic gradients to themselves without any delay, that is,
for any node $i$, $(t,i) \in \mathcal{L}_t^i$.
Further, let $\mathcal{K}_{t}^i=\bigcup_{t' \le t} \mathcal{L}_{t'}^i$. Notice that,
$\vec{v}_t^i=-\sum_{(s,j) \in\mathcal{K}_{t-1}^i} \frac{\alpha}{p} \tg\left(\vec{v}^{j}_s\right)$.
The auxiliary variable $\vec{x}_t=-\sum_{s=0}^{t-1} \sum_{j \in \mathcal{P}} \frac{\alpha}{p} \tg\left(\vec{v}^{j}_s\right)$.
This means that we use the update rule (\ref{eqn:SGDmultiplesteps}) with $I_t = \mathcal{P}$.
We proceed by proving the elastic consistency bound with $B=\frac{\tau_{max}(p-1)M}{p}$:
\begin{lemma} \label{lem:asynchmessagepassing}
	In a asynchronous message-passing system consisting of $p$ nodes and with delay bound $\tau_{max}$,we have that for any iteration $t$ and node $i$:
	\begin{equation}
	    \E \|\vec{v}_t^i-\vec{x}_t\|^2 \le \frac{\alpha^2 (p-1)^2 \tau_{max}^2 M^2}{p^2}
	\end{equation}
\end{lemma}
\begin{proof}
For any iteration $t$, let $\mathcal{C}_t=\{0,1,...,t\} \times \mathcal{P}$, so that
$\vec{x}_t=-\sum_{(s,j) \in \mathcal{C}_{t-1}} \frac{\alpha}{p} \tg\left(\vec{v}^{j}_s\right)$.
Notice that $\vec{v}_t^i-\vec{x}_t=\sum_{(s,j) \in \mathcal{C}_{t-1} \setminus \mathcal{K}_{t-1}^i} \frac{\alpha}{p} \tg ({\vec{v}_s^j})$.
If $(s,j) \in \mathcal{C}_{t-1} \setminus \mathcal{K}_{t-1}^i$, then we have that the stochastic 
gradient generated by node $j$ at iteration $s$, is still not delivered to node $i$ before iteration $t$ starts. Since each message can be delayed by $\tau_{max}$ iterations at most, we have that for any $0 \le s \le t-1-\tau_{max}$ and node $j \in \mathcal{P}$, $(s,j) \in \mathcal{K}_{t-1}^i$.
Also, we have that for any iteration $t-\tau_{max} \le s \le t-1$, $(s,i) \in \mathcal{K}_{t-1}^i$.
This means that $|\mathcal{C}_{t-1} \setminus \mathcal{K}_{t-1}^i| \le (p-1)\tau_{max}$.
We can finish the proof by using $f=(p-1)\tau_{max}$ and following exactly the same steps
as in the proof of Lemma \ref{lem:messageomissionfailures}.
\end{proof}
\vspace{-0.5em}
\paragraph{Removing the second moment bound.} Note that we can use the similar approach as in the case of crash faults and replace $M$ in the elastic consistency bound with $O(\sigma)$. More precisely, for node $i$ at iteration $t$, if  stochastic gradient $\tg(v_t^j)$ from node $j$ is delayed node $i$ uses it's own
stochastic gradient $\tg(v_t^i)$ and later corrects the error once it receives the actual gradient. See section \ref{sec:elasticscheduling} for the example where $\tau_{max}=1$.

\subsection{Variance bounded Elastic Scheduler} \label{sec:elasticscheduling}
Consider Algorithm \ref{alg:elastic}. Here we assume that discrepancy between local views
is caused by delayed stochastic gradients. But, we are guaranteed that the stochastic gradients generated 
at iteration $t$, will be delivered at iteration $t+1$ at most. This might be caused by nodes sending gradients in chunks, some chunks arrive at iteration $t$ and the rest of the stochastic gradient arrives at iteration $t+1$. We assume that stochastic gradient is considered received only when it arrives fully.
For each node $i$ at iteration $t$, if $i$ does not receive stochastic gradient generated by a node $j$
at iteration $t$, it uses it's own stochastic gradient $\tg(\vec{v}_t^i)$, instead of $\tg(\vec{v}_t^j)$.
Error caused by this substitution is corrected later during round $t+1$. The idea is to show that by elastic consistency bounds, local views of $i$ and $j$ are "close" to each other, so by using
$\tg(\vec{v}_t^i)$ instead of $\tg(\vec{v}_t^j)$ we introduce only a small error.
We again have that $\mathcal{L}_t^i$ is a set of nodes which sent stochastic gradients generated 
at iteration $t$ without any delay (Hence, $i \in \mathcal{L}_t^i$).
Auxiliary variable $\vec{x}_t=-\sum_{s=0}^{t-1}\sum_{i \in \mathcal{P}} \frac{\alpha}{p} 
\tg (\vec{v}_{s}^i)$, hence we use update rule (\ref{eqn:SGDmultiplesteps}) with $I_t=\mathcal{P}$.
For each node $i$ and iteration $t>0$, we have that $\vec{v}_t^i=\vec{x}_t+\sum_{j \in \mathcal{P} \setminus \mathcal{L}_{t-1}^i} \frac{\alpha}{p} (\tg(\vec{v}_{t-1}^i)-\tg(\vec{v}_{t-1}^j))$.
With this in place we can show that Elastic Scheduling satisfies elastic consistency bound with $B=3\sigma$.

\begin{lemma} \label{lem:elasticscheduling}
	Consider elastic scheduling model. For any node $i \in \mathcal{P}$, iteration $t$ and learning rate
	$\alpha\le \frac{1}{6L}$ (Notice that this is consistent with  upper bounds used in
	the convergence analysis), we have that 
	\begin{equation} 
	\mathbb{E}\|\vec{x}_t-\vec{v}_t^q \|^2 \le 9\sigma^2\alpha^2.
	\end{equation}
\end{lemma}
\begin{proof}
We prove the lemma using induction on the number of iterations.
Base case holds trivially since we start with 0 vectors.
For the induction step, we assume that the lemma holds for $t$.
We have that 
\begin{align*}
\E\|\vec{v}_{t+1}^i&-\vec{x}_{t+1}\|^2=\E\|\sum_{j \in \mathcal{P} \setminus \mathcal{L}_{t}^i}
\frac{\alpha}{p}(\tg(\vec{v}_{t}^i)-\tg(\vec{v}_{t}^j))\|^2 \\ &\overset{Cauchy-Schwarz}{\le} (p-1)\frac{\alpha^2}{p^2} \sum_{j \in \mathcal{P} \setminus \mathcal{L}_{t}^i} \E\|\tg(\vec{v}_{t}^i)-\tg(\vec{v}_{t}^j) \|^2\\&=\frac{(p-1)\alpha^2}{p^2} \sum_{j \in \mathcal{P} \setminus \mathcal{L}_{t}^i} \E\|\tg(\vec{v}_{t}^i)-\nabla f(\vec{v}_{t}^i)+\nabla f(\vec{v}_{t}^i)-\nabla f(\vec{v}_{t}^j)+\nabla f(\vec{v}_{t}^j)-\tg(\vec{v}_{t}^j) \|^2 \\ &\overset{Cauchy-Schwarz}{\le}
\frac{(p-1)\alpha^2}{p^2} \sum_{j \in \mathcal{P} \setminus \mathcal{L}_{t}^i} \Big(3\E\|\tg(\vec{v}_{t}^i)-\nabla f(\vec{v}_{t}^i)\|^2+3\E\|\nabla f(\vec{v}_{t}^i)-\nabla f(\vec{v}_{t}^j)\|^2 
\\&\quad\quad\quad\quad\quad\quad\quad\quad\quad\quad\quad\quad\quad\quad\quad\quad 
+3\E\|\nabla f(\vec{v}_{t}^j)-\tg(\vec{v}_{t}^j) \|^2\Big)
\\&\overset{(\ref{eqn:variance_is_bounded_assumption_f})}{\le}
\frac{(p-1)\alpha^2}{p^2} \sum_{j \in \mathcal{P} \setminus \mathcal{L}_{t}^i} \Big(6\sigma^2
+3\E\|\nabla f(\vec{v}_{t}^i)-\nabla f(\vec{v}_{t}^j)\|^2\Big) \\&\overset{(\ref{eqn:smooth_gradients_f})}{\le} \frac{(p-1)\alpha^2}{p^2} \sum_{j \in \mathcal{P} \setminus \mathcal{L}_{t}^i} \Big(6\sigma^2
+3L^2\E\|\vec{v}_{t}^i-\vec{x}_t+\vec{x}_t-\vec{v}_{t}^j\|^2\Big) \\ &\overset{Cauchy-Schwarz}{\le}
\frac{(p-1)\alpha^2}{p^2} \sum_{j \in \mathcal{P} \setminus \mathcal{L}_{t}^i} \Big(6\sigma^2
+6L^2\E\|\vec{v}_{t}^i-\vec{x}_t\|^2+6L^2\E\|\vec{x}_t-\vec{v}_{t}^j\|^2\Big) \\ &\le 
\frac{(p-1)\alpha^2}{p^2} \sum_{j \in \mathcal{P} \setminus \mathcal{L}_{t}^i} \Big(6\sigma^2
+108L^2\alpha^2\sigma^2\Big) \le \frac{(p-1)^2\alpha^2}{p^2}\Big(6\sigma^2
+108L^2\alpha^2\sigma^2\Big)\\&\le \alpha^2(6\sigma^2+108L^2\frac{1}{36L^2}\sigma^2)=9\sigma^2\alpha^2.
\end{align*}
Where, in the last steps, we used $p-1 \le p$ and assumption that lemma holds for $t$.
This gives us the proof of the lemma.
\end{proof}
\begin{algorithm}[h]
	\caption{Iteration $t$ at node $i \in \mathcal{P}$, for Elastic Scheduling} \label{alg:elastic}
	Compute $\tg\left(\vec{v}^{i}_t\right)$;~\tcp*[h]{Compute SG using the local view.}\\
	Broadcast $\tg\left(\vec{v}^{i}_t\right)$;~\tcp*[h]{Broadcast SG to all the nodes.}\\
	$\tg\gets0$;~~\tcp*[h]{Prepare to collect 
	stochastic gradients from the nodes.}\\
	\For {$j \in \mathcal{L}_t^i$}{ 
			$\tg \gets \tg + 
				\tg\left(\vec{v}^{j}_t\right)$ ~~\tcp*[h]{Add gradients received at iteration $t$.}
	}					
    \For {$j \in \mathcal{P} \setminus \mathcal{L}_t^i$}{
			$\tg \gets  \tg + \tg(\vec{v}_t^i)$ ~~\tcp*[h]{Add $\tg(\vec{v}_t^i)$ as "substitute" }
	}
	\If {t > 0} {
	    \For {$j \in \mathcal{P} \setminus \mathcal{L}_{t-1}^i$}{
			$\tg \gets  \tg - \tg(\vec{v}_{t-1}^i)+\tg(\vec{v}_{t-1}^j)$ 
			~~\tcp*[h]{correct error from the previous iteration}
	    }
	}
	$\vec{v}^i_{t+1} \gets \vec{v}^i_{t} - \frac{\alpha}{p}\tg$;~\tcp*[h]{
	Update \modl vector.}
\end{algorithm}

\subsection{Shared-Memory Systems}\label{sec:sm}

We consider a shared-memory system with $p$ processors $\mathcal{P} = \{1, 2, \ldots, p\}$ that supports atomic \texttt{read} and \fadd~(\texttt{faa}). The \modl vector $\vec{x} \in \mathbb{R}^d$ is shared by the processes for concurrent lock-free read and write or update. The read/update at each of the indices $\vec{x}[i]$, $1\le i \le d$, of \modl, are atomic.  By design, each process reads as well as writes over an inconsistent snapshot of $\vec{x}$, see~\cite{alistarh2018convergence}. We order the iterations of SGD by the atomic \texttt{faa} over the first index of $\vec{x}$. Note that, iterations by the processes are collectively ordered. Let $q$ be the process, which computes stochastic gradient at iteration $t$, the inconsistent view (due to asynchrony) which is read by $q$ at iteration $t$ is denoted by $\vec{v}_{t}^q$.
See Algorithm \ref{alg:sm} for the formal description.
\begin{algorithm}[h]
	\caption{Iteration $t$, processor $q$.}\label{alg:sm}
	\lFor(\tcp*[h]{Lock-free read.}){$1\le i \le 
	d$}{$\vec{v}_{t}^q[i]\gets\texttt{read}(\vec{x}[i])$}
	Compute $\tg\left(v_t^q\right)$\;
	\lFor(\tcp*[h]{Lock-free update.}){$1\le i \le 
	d$}{$\texttt{faa}(\vec{x}[i],\alpha\tg\left(\vec{v}_{t}^q\right)[i])$}
\end{algorithm}
In this case, the auxiliary variable $\vec{x}_t$ corresponds to the sum of first $t$ stochastic gradients (according to the ordering described above) multiplied by $-\alpha$. Hence, we use the update rule (\ref{eqn:SGDsinglestep}): $\vec{x}_{t+1}=\vec{x}_t-\alpha \tilde G(\vec{v}_t^q)$.

At iteration $t$, let $\tau_{t}^i$ be the delay in the stochastic gradient update at an arbitrary index $1 \le i \le d$, which  essentially means that $\vec{v}_t^q[i]=\vec{x}_{t-\tau_{t}^i}[i]$. Let $\tau_t=\max\{\tau_{t}^1, \tau_{t}^2, ..., \tau_{t}^d \}$. It is standard to assume that  for any $t \ge 0$, $\tau_t$ is upper bounded by $\tau_{max}$ \cite{hogwild,lian2015asynchronous,desa2015hogwild,alistarh2018convergence}. Drawing from the literature of shared-memory distributed computing, considering an iteration as an operation, $\tau_{max}$ is essentially the maximum number of concurrent operations during the lifetime of any operation in the system, which is defined as \textit{interval contention} \cite{afek1999long}. With these specifications in place, Lemma \ref{lemma:smconsistency} shows that the shared-memory asynchronous SGD scheme satisfies elastic consistency with 
$B=\sqrt{d}\tau_{max}M$. 
\begin{lemma}
	\label{lemma:smconsistency}
	Given an asynchronous shared-memory system with maximum delay bound $\tau_{\max}$, we have 
	that for processor $q$ which generates stochastic gradient at iteration $t$.
	$\mathbb{E}\|\vec{x}_t-\vec{v}_t^q \|^2  \le d \tau_{max}^2 \alpha^2 M^2$
\end{lemma}
\begin{proof}
	The $1$-norm of the difference between the inconsistent snapshots 
	$\vec{x}_t$ and $\vec{v}_t$ is bounded as follows:
	\begin{align*}
		&\norm{\vec{x}_t - \vec{v}_{t}^q}_1 \le 
		\sum_{j=1}^{\tau_{max}}\norm{\vec{x}_{t-j+1} - \vec{x}_{t-j}}_1\\
		&\le 
		\sqrt{d}\sum_{j=1}^{\tau_{max}}\norm{\vec{x}_{t-j+1} - 
			\vec{x}_{t-j}}~(\text{as} 
		~\norm{\vec{x}}_1\le\sqrt{d}\norm{\vec{x}},~\mathrm{for}~\vec{x}\in\mathbb{R}^d).
	\end{align*}
Then,
	\begin{align*}
		\| \vec{x}_t - \vec{v}_t^q \|^2 &\le 
	    \| \vec{x}_t - \vec{v}_t 
		\|_1^2~(\mathrm{using} 
		\norm{\vec{x}}\le\norm{\vec{x}}_1,~\mathrm{for}~\vec{x}\in\mathbb{R}^d)\\
		&\le 
		\Big(\sqrt{d}\sum_{j=1}^{\tau_{max}}\|\vec{x}_{t-j+1} - \vec{x}_{t-j}\|
			\Big)^2\\
		&\overset{Cauchy-Schwarz}{\le} 
		d\tau_{max}\sum_{j=1}^{\tau_{max}}\|\vec{x}_{t-j+1} - \vec{x}_{t-j}\|^2.
	\end{align*}

Fix $t$ and $1 \le j \le \tau_{max}$. Recall that $\vec{x}_{t-j+1} - \vec{x}_{t-j}=-\alpha \tg(\vec{v}_{t-j}^{r})$,
for some processor $r$, which computed stochastic gradient at iteration $t-j$.
Hence, 
\begin{equation}
\E\|\vec{x}_{t-j+1} - \vec{x}_{t-j}\|^2 = \alpha^2 \E \|\tg(\vec{v}_{t-j}^{r}) \|^2 \overset{(\ref{eqn:grad_is_bounded_assumption_f})}{\le} \alpha^2M^2.
\end{equation}
Thus, by combining the above two inequalities we get that:
\begin{equation*}
\E\| \vec{x}_t - \vec{v}_t^q \|^2 \le d\tau_{max}\sum_{j=1}^{\tau_{max}}\E\|\vec{x}_{t-j+1} - \vec{x}_{t-j}\|^2 \le d \tau_{max}^2 \alpha^2 M^2.
\end{equation*}

\end{proof}

\subsection{Communication-Efficient Methods}\label{sec:cr}
In this section, we consider a \textit{synchronous} message-passing system of $p$ nodes $\mathcal{P} = \{1, 2, \ldots, p\}$ executing SGD iterations. 
For simplicity, we assume that the system is \emph{synchronous and fault-free}, 
In each iteration nodes  broadcast a \textit{compressed} version of the stochastic gradient, they computed, in order to reduce communication costs. Communication-efficiency can be achieved in two ways: for a computed stochastic gradient, (a) quantization: broadcast a quantized vector that requires fewer bits~\cite{seide2014sgd1bit, QSGD}, or, (b) sparsification: broadcast a sparse vector~\cite{TopK,StichCJ18}. 


Clearly, this strategy leads to losses in the \modl updates at each iteration. A popular approach to control this error is using \textit{error-feedback}: the accumulated residual error from the previous broadcasts is added to the stochastic gradient at the current iteration before applying quantization or sparsification. We will show that the residual error can be modelled in the context of elastic consistency, and that it stays bounded in the case of popular communication-efficient techniques, which implies their convergence. 

In the description below, we will be using the term \textit{lossy compression} collectively for both quantization and sparsification. A vector $\vec{x}$ which undergoes a lossy compression will be called a \textit{compressed} vector, it's compressed value will be denoted by $Q(\vec{x})$.

\begin{algorithm}[h]
	\caption{Iteration $t$ at a node $i \in \mathcal{P}$}\label{alg:cr}
	Compute $\tg\left(\vec{v}^{i}_t\right)$\;
	Compute $\vec{w}^{i}_{t} \gets 
	\vec{\epsilon}^{i}_{t}+\alpha\tg(\vec{v}^{i}_{t})$;~\tcp*[h]{Add the 
	accumulated error to the computed SG.}\\\label{algline:errfeed}			
	Compute and Broadcast $Q(\vec{w}^{i}_{t})$;~\tcp*[h]{Apply lossy 
	compression and broadcast.}\\\label{algline:brdcast}
	Compute 
	$\vec{\epsilon}^i_{t+1}\gets\vec{w}^{i}_{t}-Q(\vec{w}^{i}_{t})$;~\tcp*[h]{Update
	 the accumulated error.}\\\label{algline:erradd}			
	$\tg \gets0$\;
	\For{each 
	$j \in \mathcal{P}$ (including $i$)} {
		Receive $Q(\vec{w}^{j}_{t})$;~$\tg \gets \tg + Q(\vec{w}^{j}_{t})$\;
	}	
	$\vec{v}^i_{t+1} \gets \vec{v}^i_{t} - 
	\frac{\tg}{p}$;~\tcp*[h]{Update \modl 
	vector.}\\\label{line:cr-update}
\end{algorithm}
\

A typical iteration at a node $i$ is shown in Algorithm \ref{alg:cr}. The algorithm works as follows. Each node $i$ maintains a local model $\vec{v}_t^i$, and a local error accumulation $\epsilon_t^i$, starting from $\vec{v}_t^i = \vec{0}^d = \vec{\epsilon}_t^i$. 

At each iteration, a node $i$ computes a stochastic gradient with respect to its local view $\vec{v}_t^i$, adds to it the accumulated error from the previous iterations, applies a lossy compression function $Q$ to the result, and broadcasts to the \modl servers. The error accumulation is updated to reflect the lossy compression, see Line \ref{algline:erradd}.
The compression $Q$ provably satisfies the following:
\begin{equation}\label{eqn:lossy-commpression} \|Q(\vec{w})-w\|^2 \le 
\gamma\|\vec{w}\|^2,~\text{$\forall \vec{w} \in \R^d$, and  $0\le\gamma<1$}.
\end{equation}
In the above inequality parameter $c$, depends on the compression scheme that we use.

The error-feedback strategy particularly helps in asymptotically compensating 
for the biased stochastic gradient updates. However, if the compression scheme 
is unbiased, e.g. QSGD \cite{QSGD}, the convergence theory works even without 
the error-feedback. For such a method, line \ref{algline:errfeed} in Algorithm 
\ref{alg:cr} changes to $\vec{w}^i_{t} \gets \alpha\tg(\vec{v}^i_{t})$; all other 
steps remain unchanged. Our discussion in this sub-section focuses on methods 
with error-feedback.

The local view of \modl $\vec{v}_t^i$ is inconsistent in the sense that it is updated with compressed stochastic gradients. To model this in the context of elastic consistency, we define the auxiliary parameter $\vec{x}_t$, 
as sum of all stochastic gradients generated by the algorithm up to and excluding iteration $t$, multiplied by $-\frac{\alpha}{p}$.
Hence, we use the update rule (\ref{eqn:SGDmultiplesteps}) with $I_t=\{1,2,...,p\}=\mathcal{P}$:
\begin{equation}
\vec{x}_{t+1}=\vec{x}_t-\frac{\alpha}{p}\sum_{i \in \mathcal{P}} \tg(\vec{v}_t^i).
\end{equation}
The local view of each node $q \in P$ is updated as:
\begin{equation}
\vec{v}_{t+1}^q=\vec{v}_t^q-\frac{1}{p}\sum_{i \in \mathcal{P}} Q(\alpha\tg(\vec{v}_t^i)+\epsilon_t^i).
\end{equation}

Using induction on the number of iterations, it is easy to show that for any $t \ge 0$ and node $q$:
\begin{equation} \label{eqn:xminusv}
\vec{v}_t^q-\vec{x}^t=\frac{1}{p}\sum_{i \in \mathcal P} \epsilon_t^i.
\end{equation}

Thus, using the above equation we can derive the elastic consistency bounds.
\begin{lemma}
For any node $q$ and iteration $t \ge 0$. We have that 
\begin{equation}
\E\|\vec{x}_t-\vec{v}_t^q\|^2 \le \frac{(2-\gamma)\gamma M^2\alpha^2}{(1-\gamma)^3}.
\end{equation}
\end{lemma}
\begin{proof}

we start with
\begin{align*}
\sum_{i \in \mathcal{P}} \|\epsilon_{t+1}^i\|^2 &= \|Q(\alpha\tg(\vec{v}_t^i)+\epsilon_t^i)- (\alpha\tg(\vec{v}_t^i)+\epsilon_t^i)\|^2 \overset{(\ref{eqn:lossy-commpression})}{\le} \gamma \sum_{i \in \mathcal{P}} \|\alpha\tg(\vec{v}_t^i)+\epsilon_t^i\|^2 \\&\overset{Young}{\le} (1+1-\gamma)\gamma 
\sum_{i \in \mathcal{P}} \|\epsilon_{t}^i\|^2+(1+\frac{1}{1-\gamma})\alpha^2\gamma\sum_{i \in \mathcal{P}}
\|\tg(\vec{v}_t^i)\|^2
\end{align*}
\end{proof}
Next, using induction on the number of iterations, we upper bound $\sum_{i \in \mathcal{P}} \E\|\epsilon_{t}^i\|^2$ by $\frac{(2-\gamma)\gamma M^2\alpha^2p}{(1-\gamma)^3}$.
\\The base case holds trivially. For the induction step we assume that 
the statement holds for $t$. We get that:
\begin{align*}
\sum_{i \in \mathcal{P}} &\E\|\epsilon_{t+1}^i\|^2 {\le} (1+1-\gamma)\gamma 
\sum_{i \in \mathcal{P}} \E\|\epsilon_{t}^i\|^2+(1+\frac{1}{1-\gamma})\alpha^2\gamma\sum_{i \in \mathcal{P}}
\E\|\tg(\vec{v}_t^i)\|^2 \\&\overset{(\ref{eqn:grad_is_bounded_assumption_f})}{\le}(1+1-\gamma)\gamma 
\sum_{i \in \mathcal{P}} \E\|\epsilon_{t}^i\|^2+(1+\frac{1}{1-\gamma})\alpha^2\gamma M^2p \\ &\le \gamma\alpha^2M^2p\Bigg(\frac{(2-\gamma)^2\gamma}{(1-\gamma)^3}+1+\frac{1}{1-\gamma} \Bigg)=\frac{(2-\gamma)\gamma M^2\alpha^2p}{(1-\gamma)^3}.
\end{align*}
This gives us that 
\begin{align*}
\E\|\vec{x}_t-\vec{v}_t^q\|^2\overset{(\ref{eqn:xminusv})}{=}\E\|\frac{1}{p}\sum_{i \in \mathcal P} \epsilon_t^i\|^2\overset{Cauchy-Schwarz}{\le} \frac{1}{p} \sum_{i \in \mathcal{P}} \E\|\epsilon_{t}^i\|^2 \le \frac{(2-\gamma)\gamma M^2\alpha^2}{(1-\gamma)^3}.
\end{align*}

In the following, we show that TopK and One-Bit quantization schemes satisfy elastic consistency bounds.

\paragraph{TopK quantization}
The \textit{TopK} algorithm~\cite{strom2015scalable} presents a sparsification 
scheme for the stochastic gradients. Essentially, we select the top 
$K$ of the indices of $\vec{w}$ sorted by their absolute value. Because $d-K$ 
indices of a vector, which are not the top ones by their absolute value, are 
discarded in this method, it clearly satisfies inequality 
(\ref{eqn:lossy-commpression}) for $\gamma=\frac{d-K}{d}$. 
Hence, TopK quantization satisfies elastic consistency bound with constant
$B=\sqrt{\frac{(1-K/d)(1+K/d)}{(K/d)^3}}M=\sqrt{\frac{d}{K}(\frac{d^2}{K^2}-1)}M$.

%


\paragraph{One-Bit quantization} 
The one-bit SGD quantization was first described in~\cite{seide2014sgd1bit}. 
We use the notation $[\vec{x}]_i$ to denote the $i$'th component of $\vec{x}$.
Consider the vector $\vec{w}$ and let $S^+(\vec{w})$ be its index of positive 
components and $S^-(\vec{w})$ that of its negative components,
i.e., $S^+(\vec{w}) = \{i:\, [\vec{w}]_i\ge 0\}$ and $S^-(\vec{w}) = \{i:\, 
[\vec{w}]_i<0\}$. Then let $\bar{w}^+=\frac{1}{|S^+(\vec{w})|}
\sum_{i\in S^+(\vec{w})} [\vec{w}]_i$ and $\bar{w}^-=\frac{1}{|S^-(\vec{w})|}
\sum_{i\in S^-(\vec{w})} [\vec{w}]_i$.
Now define the one-bit quantization operation $Q(\cdot)$ as,
\begin{equation}
\label{eqn:onebit}
[Q(\vec{w})]_i = \left\{ \begin{array}{lr} \bar{w}^+ & \text{ for }i\in 
S^+(\vec{w}). \\
\bar{w}^- & \text{ for }i\in S^-(\vec{w}).\end{array}\right.
\end{equation}
It  easy to verify that one-bit quantization satisfies inequality 
(\ref{eqn:lossy-commpression}) for $\gamma=1-\frac{1}{d}$. Hence, One-bit quantization
satisfies elastic consistency bounds with $B=\sqrt{d(d^2-1)}M$.

\input{append_experiments}

%% file: append_experiments.tex
\section{Additional Experiments}
\begin{figure}[ht]
    \centering
    \begin{minipage}[b]{0.45\linewidth}
        \centering
        \includegraphics[width=7cm]{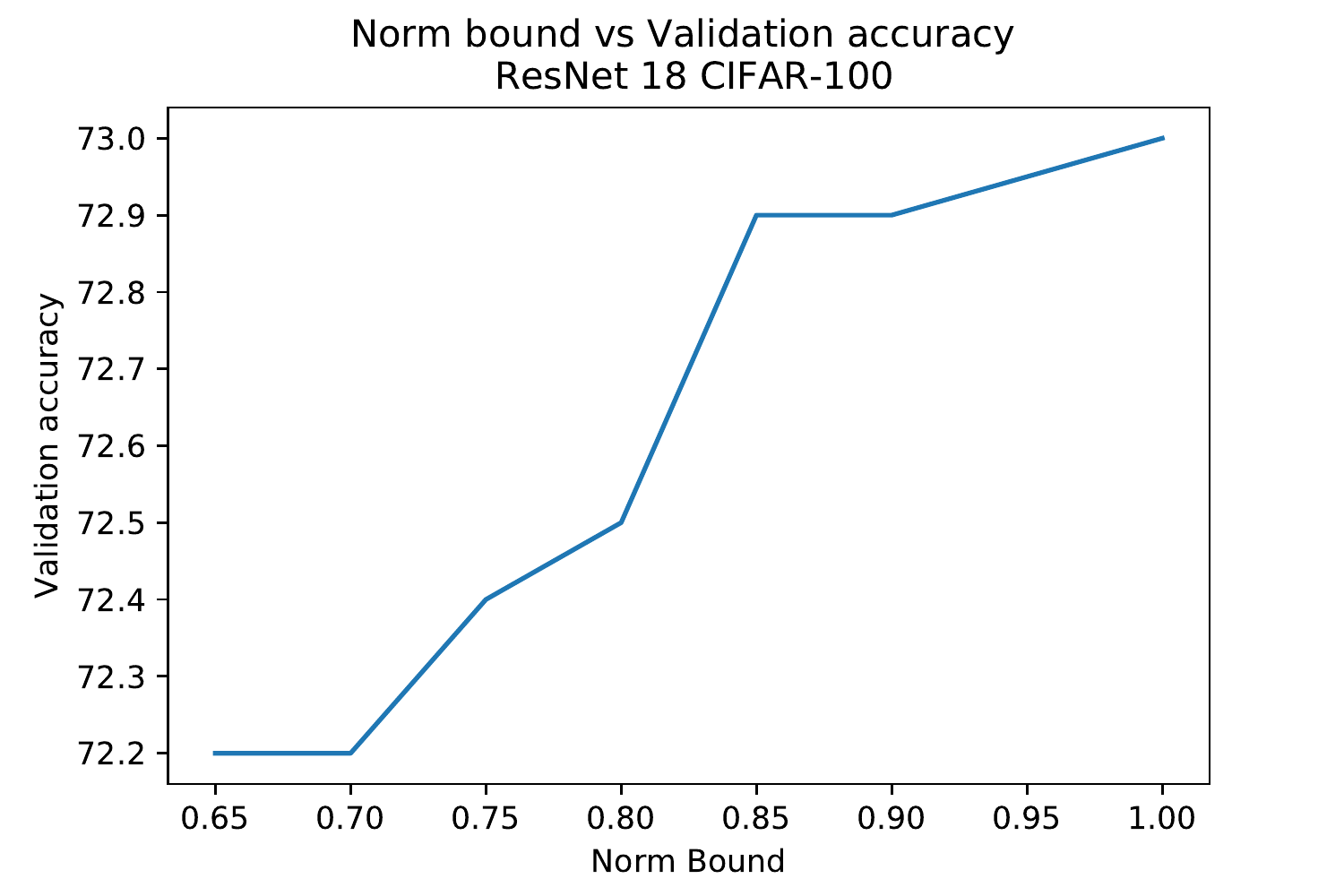}
    \end{minipage}
    \quad
    \begin{minipage}[b]{0.45\linewidth}
        \centering
        \includegraphics[width=7cm]{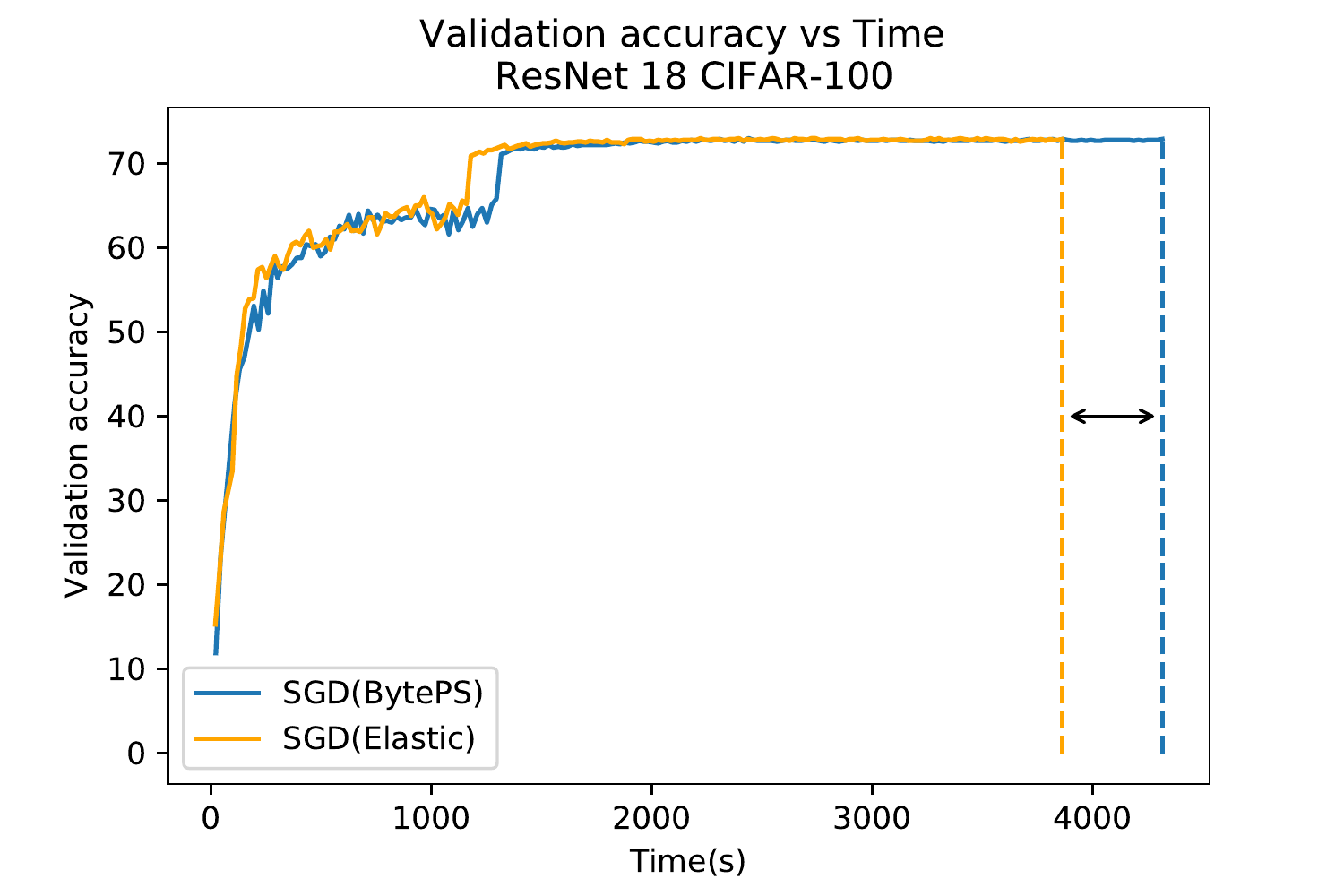}

    \end{minipage}
    \caption{Elastic bound-v-accuracy (\textbf{left}) and accuracy-v-time (\textbf{right}) for RN18 on CIFAR-100.}
    \label{fig:exp1}
\end{figure}

\begin{figure}[ht]
    \centering
    \begin{minipage}[b]{0.45\linewidth}
        \centering
        \includegraphics[width=7cm]{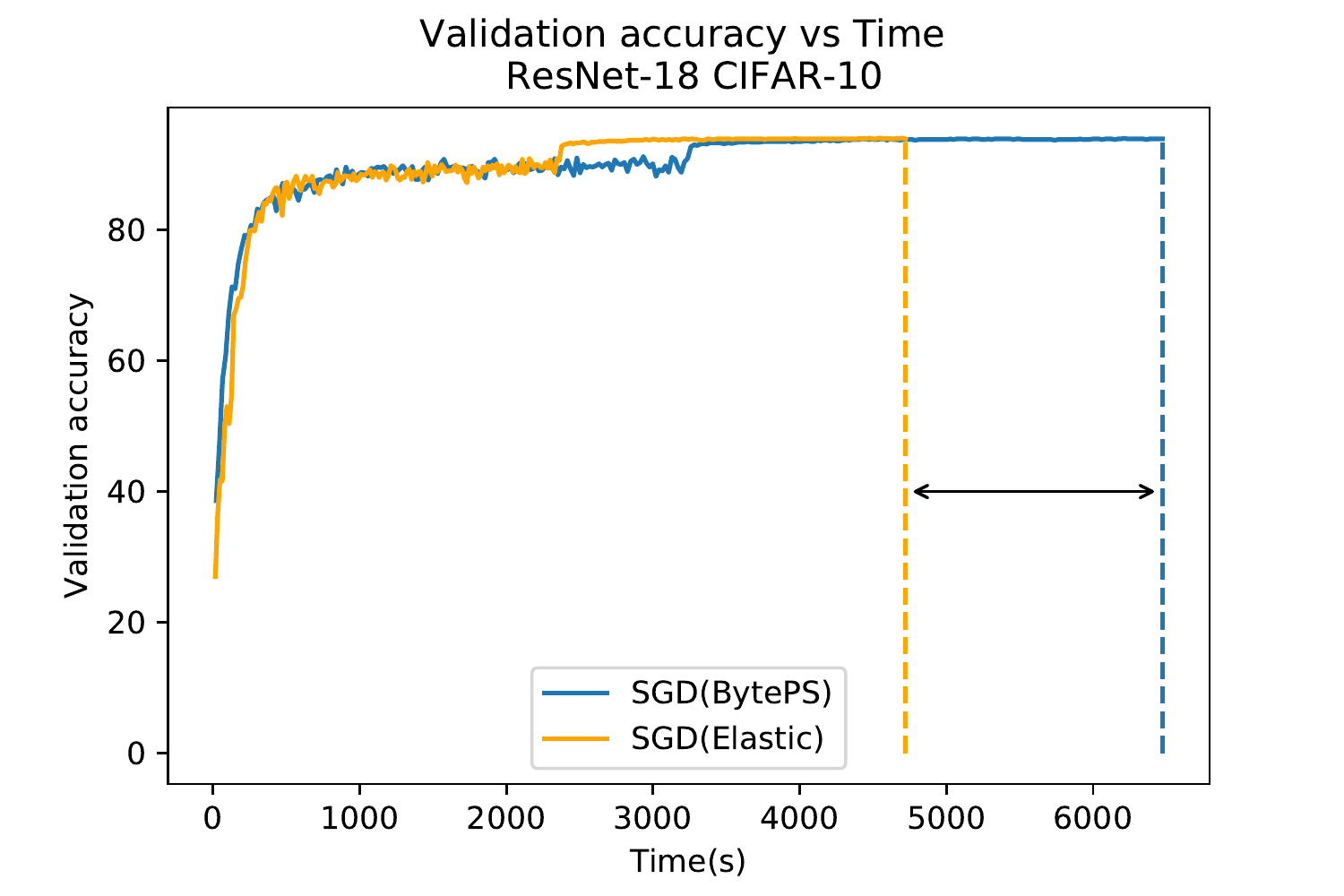}
    \end{minipage}
    \quad
    \begin{minipage}[b]{0.45\linewidth}
        \centering
        \includegraphics[width=7cm]{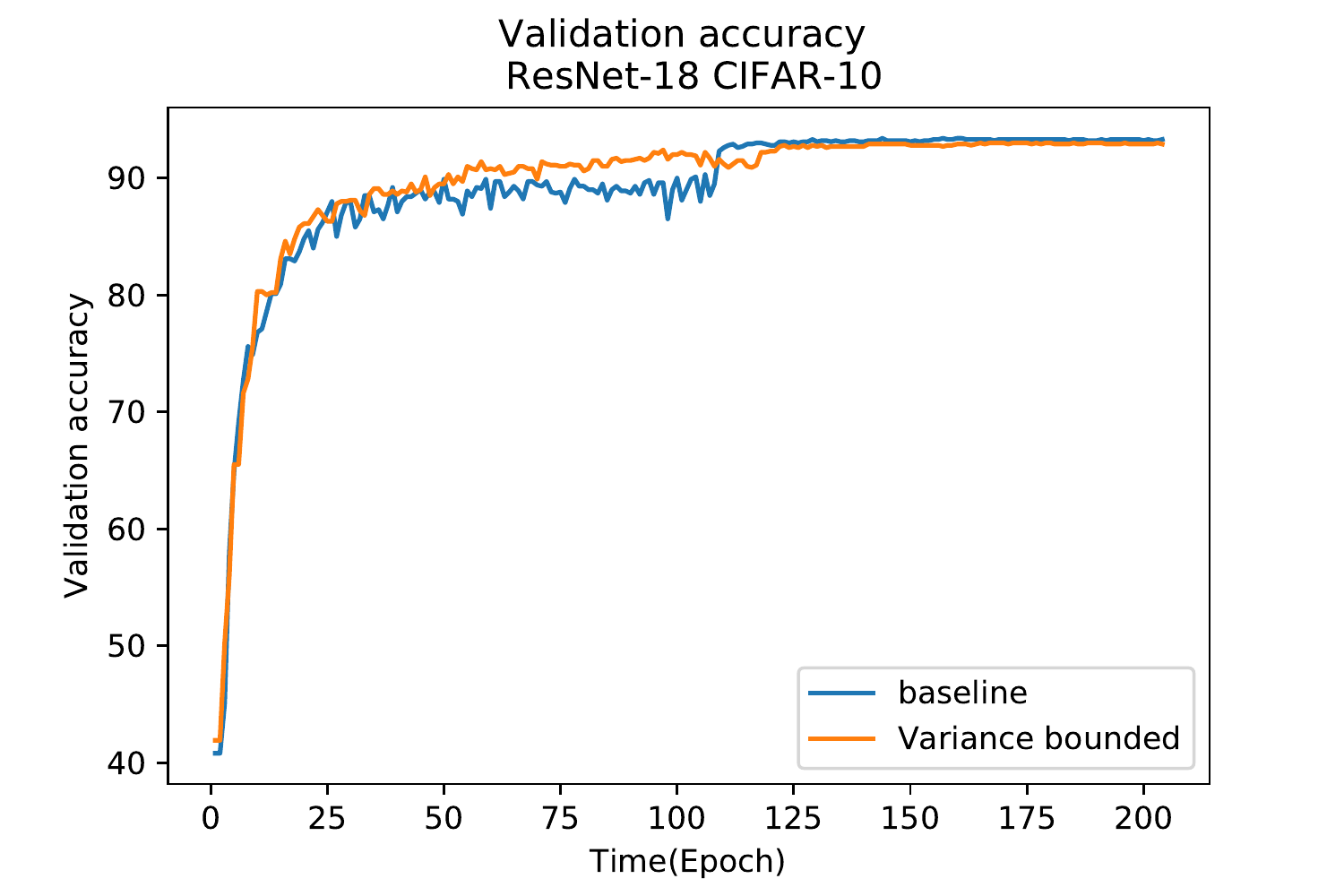}
    \end{minipage}
    \caption{Accuracy-v-time for norm-bounded scheduler (\textbf{left}) and Accuracy-v-epoch for variance-bounded (\textbf{right}) scheduler for RN18 on CIFAR-10.}
    \label{fig:exp2}
\end{figure}

We empirically validate our approach on data-parallel GPU training of deep neural networks.

\paragraph{Setup.} We performed experiments on 2 AWS EC2 P3.2xlarge instances each with a V100 GPU. For reproducibility, we controlled network latency via the \textit{tc} utility from the iproute2 network utilities collection. In all experiments, network latency was set to 5ms with 0.2ms jitter sampled from a normal distribution, which we found to be realistic values.

\paragraph{Implementation.}
We implemented the elastic scheduler on top of the Horovod \cite{sergeev2018horovod} and BytePS\cite{byteps} distributed frameworks. This allows us compare our approach to the BytePS baseline, which implements a so-called \emph{cross-barrier} strategy~\cite{byteps}, which can reorder/prioritize packets as long as perfect consistency is preserved.
As described, the idea of our implementation is tracking the norm ratio of the synchronized parts of the model and starting the next forward call of a layer in case if one of two conditions is satisfied: either the layer is synchronized, or the ratio reaches the $\beta$ threshold.

\paragraph{Accuracy and Speedup.}
Figure~\ref{fig:exp} and Figure~\ref{fig:exp1} show the validation accuracy for WideResnet28x8~\cite{zagoruyko2016wide} and ResNet18~\cite{he2016deep} on the CIFAR100 dataset \cite{krizhevsky2009learning}, respectively, for the \emph{norm-bounded} scheduler. Both models were trained with standard hyperparameters: minibatch size 128, base learning rate 0.1, decaying at 60th, 120th and 160th epochs, for a total of 200 epochs.

Figure~\ref{fig:exp2}(left) gives the results for ResNet18 on CIFAR-10 where we can see the full accuracy recovery for the norm-bounded scheduler with $\beta=0.1$ and $\sim 30\%$ speedup.

Finally, we examine the convergence of the \emph{variance-bounded} elastic scheduler for ResNet18 on CIFAR-10. Results are given on the Figure~\ref{fig:exp2}(right).
The CIFAR-10 experiments were conducted with 256 minibatch size, 0.1 base learning rate. Variance-bounded experiments were run without momentum, while all other experiments had momentum equal to 0.9.